\documentclass[twoside]{article}
\usepackage[dvipsnames]{xcolor}
\usepackage{graphicx}
\usepackage{amsthm}
\usepackage{amsmath}
\usepackage{mathtools}
\usepackage{mathrsfs}
\usepackage{amssymb}
\usepackage{etoolbox}
\usepackage[linesnumbered,lined,boxed,ruled]{algorithm2e}
\usepackage[section]{placeins}
\usepackage[breaklinks,colorlinks,citecolor=blue,linkcolor=blue,urlcolor=blue]{hyperref}
\usepackage[all]{hypcap}
\usepackage{cancel}
\usepackage{csquotes}
\usepackage[toc,page]{appendix}
\usepackage{multirow}
\usepackage{datetime}
\usepackage{times}
\usepackage[format=plain, labelfont={bf,it}, textfont=it]{caption}
\usepackage{subcaption}
\usepackage{enumitem}
\usepackage{tcolorbox}
\usepackage{natbib, bibentry}
\usepackage{orcidlink}
\usepackage{fancyhdr}
\usepackage{listings}
\usepackage{scalerel}
\usepackage{stackengine,wasysym}

\newdateformat{mdydate}{\monthname[\THEMONTH] \THEDAY, \THEYEAR}
\newdateformat{monthyeardate}{\monthname[\THEMONTH] \THEYEAR}
\newdateformat{vmonthyeardate}{\THEYEAR.\THEMONTH.\THEDAY}

\newcommand{\be}{\begin{equation}}
\newcommand{\ee}{\end{equation}}
\newcommand{\bea}{\begin{eqnarray}}
\newcommand{\eea}{\end{eqnarray}}

\newtheorem{proposition}{Proposition}

\usepackage{geometry}
\graphicspath{{figures/}}

\newcommand{\volume}{3}
\newcommand{\firstpage}{33}
\newcommand{\lastpage}{54}
\newcommand{\yyyy}{2026}
\newcommand{\mm}{September}
\newcommand{\dd}{28}
\newcommand{\authors}{Joel Pfeffer et al.}
\newcommand{\fulltitle}{Introducing the CZAR Loss: A Tailored Objective Function for Financial Log-Return Predictions}
\newcommand{\shorttitle}{CZAR loss function}

\newcommand{\doi}{10.70235/allora.0x\volume\ifnum\numexpr\firstpage<10 000\else\ifnum\numexpr\firstpage<100 00\else\ifnum\numexpr\firstpage<1000 0\fi\fi\fi\firstpage}
\fancypagestyle{firstpage}{%
    \fancyhead[L]{\footnotesize Allora Decentralized Intelligence \textbf{\volume}, \firstpage--\lastpage; \yyyy\ \mm\ \dd}
    \fancyhead[R]{\footnotesize doi:\href{https://doi.org/\doi}{\textcolor{blue}{\doi}}}
    \fancyfoot{}
    
}

\AtBeginDocument{\thispagestyle{firstpage}}

\begin{document}

\title{\fulltitle}
\author{\authors}
\date{\monthyeardate{\today}}

% START EDITING HERE

\vskip30mm
\begin{center}
\begin{minipage}{170mm}
\begin{center}
\vskip5mm
{\fontsize{15pt}{15pt}\textbf{Introducing the CZAR Loss: A Tailored Objective Function for Financial Log-Return Predictions}}
\vskip5mm
Joel Pfeffer$^{\orcidlink{0000-0003-3786-8818}}$$^{1}$,
J.~M.~Diederik Kruijssen$^{\orcidlink{0000-0002-8804-0212}}$,$^{1}$
Florian Stecker$^{\orcidlink{0000-0002-7687-5116}}$$^{1}$ \&
Steven~N.~Longmore$^{\orcidlink{0000-0001-6353-0170}}$$^{1,2}$
\vskip1mm
$^{1}$\textit{Allora Foundation},
$^{2}$\textit{Liverpool John Moores University}
\end{center}
\end{minipage}
\end{center}
\vspace{3mm}

\begin{abstract}
\noindent
In quantitative finance, standard regression losses are misaligned with the economics of return prediction.
As the conditional mean of financial log-returns is close to zero, symmetric losses such as the mean squared and mean absolute errors make the constant zero forecast a near-optimal solution, penalizing models with genuine but noisy directional skill. This applies both during training, where predictions shrink toward zero, and during evaluation, where trivial forecasters can lead loss-based rankings.
Under a Gaussian linear prediction model, we show that \textit{all} symmetric monotonic losses share a universal breakeven directional accuracy against the zero predictor, which rises sharply and becomes unobtainable as the prediction noise approaches the standard deviation of the returns.
We introduce the CZAR (Composite Zero-Agnostic Return) loss function, a piecewise quadratic loss built around five requirements derived from this analysis: convexity in the prediction, asymmetry oriented by the direction of the true return that vanishes at zero, near-linear penalization of undershoots and wrong-direction predictions, divergence for large errors, and an adaptive loss floor for evaluation.
CZAR is provably convex in the prediction at fixed true value, has closed-form gradient and Hessian suitable for custom objectives in gradient-boosted libraries, and its four hyperparameters reduce to a single choice through correlated defaults.
In idealized tests, the minimum directional accuracy required for a CZAR-evaluated forecaster to outperform the zero predictor under mean log loss remains near the $50\%$ chance level, whereas the corresponding threshold for symmetric losses rises sharply with prediction noise.
This advantage persists under heavy-tailed return distributions.
In a LightGBM experiment on intraday (15-minute and 1-hour) BTC log-returns, CZAR-trained models reduce the `zero-returns bias' of the L1 and L2 baselines and improve long--short performance and directional accuracy on large-magnitude returns.
\end{abstract}
\vspace{3mm}

% Start of the body of the paper

\section{Introduction} \label{sec:intro}

Predicting financial log-returns is an unusually hostile regression problem.
The predictable component of returns is small relative to the noise, return distributions are heavy-tailed \citep{blattberg74,gopikrishnan99,cont01}, and at most horizons the conditional mean return is close to zero.
Machine-learning models, most prominently gradient-boosted decision trees \citep{chen16,ke17} and deep networks \citep{gu20}, are nevertheless routinely trained on return targets using generic regression losses such as the mean squared error (MSE) or mean absolute error (MAE).
These point-error objectives are the default in general forecasting practice \citep{makridakis20,makridakis22}, yet they are designed for settings where the target carries substantial extractable signal.
The choice of loss function matters twice over.
During training, its gradient and curvature determine what the model learns; during evaluation, sample-averaged losses drive hyperparameter tuning, early stopping, and model selection.
A loss function misaligned with the actual prediction objective therefore corrupts both the models that are trained and the way they are compared.

For predicting log-returns, symmetric losses are misaligned in a specific and severe way, which we refer to as a \textit{zero-returns bias}.
Symmetric losses are minimized in expectation by central tendencies of the conditional return distribution (the conditional mean for MSE, the conditional median for MAE; \citealt{gneiting11}), and for log-returns these are close to zero.
A constant zero prediction therefore achieves a near-minimal expected loss, while a model with genuine directional skill but realistic prediction noise is penalized for the noise more than it is rewarded for the signal.
As we show in \S\ref{sec:background}, under a Gaussian linear prediction model a forecaster with directional accuracy well above chance can incur a \textit{higher} MAE and MSE than the trivial zero predictor, and the directional accuracy required to break even against zero returns follows a universal curve shared by \textit{all} symmetric monotonic losses, rising sharply as the prediction noise approaches the scale of the returns themselves.
This problem is not confined to MSE and MAE.
The conditional expectation is the optimal predictor for every loss in the Bregman family \citep{banerjee05}, so any such loss makes near-zero prediction an attractor whenever the conditional mean return is near zero.
Practically, the predictable component of returns recovered by machine-learning models is small, with out-of-sample $R^2$ values of at most a few percent \citep[e.g.][]{gu20,leippold22}.
A model trained with a symmetric loss responds by shrinking its predictions toward zero, a degenerate solution that minimizes training loss while destroying the forecast as a directional signal.
However, directional information is precisely what is economically valuable.
Forecast profitability is tied to directional accuracy rather than to conventional error measures \citep{leitch91,granger00}, directional accuracy is a standard evaluation criterion with well-developed tests \citep{pesaran92}, and even statistically weak return predictability can be economically significant \citep{kandel96}.

Asymmetric loss functions are the natural remedy, and they have a long history.
\citet{granger69} showed that a non-quadratic cost of error renders the optimal forecast biased away from the conditional mean; \citet{zellner86} introduced the LINEX loss; \citet{christoffersen97} derived the optimal (time-varying) bias under asymmetric loss and conditional heteroskedasticity; and \citet{elliott05} and \citet{patton07} developed estimation and testing under flexible families of asymmetric loss.
On the training side, the pinball loss of quantile regression \citep{koenker78} and the asymmetric squared loss of expectile regression \citep{newey87} are the canonical asymmetric objectives, and the Huber loss \citep{huber64} the canonical robust one.
None of these resolve the zero-returns bias, for two reasons.
First, their asymmetry is fixed and oriented by the sign of the prediction error.
They uniformly prefer over- (or under-) prediction everywhere, which merely relocates the optimal predictor from the conditional mean to a fixed quantile or expectile of the conditional distribution \citep{gneiting11}.
This results in a uniform bias rather than a removal of the attractor, and one that imposes a directional preference even when the true return is zero, where none is justified.
What return prediction requires is instead asymmetry oriented by the \textit{direction of the true return}.
Overshooting a correctly-signed prediction is nearly harmless, while undershooting or predicting the wrong direction is not, and the asymmetry should vanish as the true return approaches zero and strengthen as its magnitude grows.
Second, as we show in \S\ref{sec:background}, a constant asymmetry strong enough to materially lower the breakeven directional accuracy introduces pathologies of its own.
Under log-averaged evaluation (ranking by the mean of the logarithm of the losses, equivalently their geometric mean, an established choice for aggregating positive error measures that span orders of magnitude, as it limits domination by large-magnitude samples \citep{fleming86,armstrong92,hyndman06}) it can push the breakeven below $50\%$, so that a model that is systematically \textit{wrong} about direction outranks the zero predictor.

A complementary strand of work abandons statistical losses altogether in favor of economic objectives.
This strand includes trading systems trained by direct maximization of the Sharpe ratio \citep{moody01,zhang20}, decision-focused learning that trains the predictor on the downstream optimization objective \citep{elmachtoub22}, no-arbitrage moment conditions used as adversarial losses \citep{chen24}, and bespoke asymmetric or direction-aware losses for asset-return prediction \citep{dessain23,michankow24}.
This literature confirms that replacing symmetric statistical losses with finance-aware objectives improves trading-relevant performance, but the resulting objectives are typically non-convex, tied to a particular portfolio construction, or lacking the analytical gradient and Hessian required of a custom objective in gradient-boosted libraries \citep{chen16,ke17}.
What is missing is a general-purpose loss for return prediction that is convex in the prediction, directly trainable, and constructed explicitly so that neither the zero predictor nor a wrong-direction predictor can outrank a genuinely informative model.

In this paper we introduce the CZAR (Composite Zero-Agnostic Return) loss function, which is designed to fill this gap.
Our contributions are as follows.
(i)~From an analysis of symmetric and asymmetric losses under a Gaussian linear prediction model, we derive five requirements that a loss function adequate for return prediction must satisfy: convexity in the prediction, direction-oriented asymmetry that is symmetric at a zero true return, near-linear penalization of undershoots and wrong-direction predictions, divergence for large errors, and an adaptive loss floor near the mean return (\S\ref{sec:background}).
(ii)~We construct CZAR, a piecewise quadratic loss with an asymmetry that adapts to the magnitude of the true return and an evaluation-time loss floor, satisfying all five requirements, and reduce its four hyperparameters to effectively one through correlated defaults (\S\ref{sec:czar}).
(iii)~We prove that CZAR is convex in the prediction at fixed true value (\S\ref{sec:czar:derivatives}).
(iv)~We derive the analytical gradient and Hessian, making CZAR directly usable as a training objective in gradient-boosted frameworks (\S\ref{sec:czar:derivatives}, with a reference implementation in Appendix~\ref{app:code}).
(v)~In idealized tests, we show that CZAR's log-loss breakeven directional accuracy remains near the $50\%$ chance level as prediction noise increases, whereas the corresponding thresholds for symmetric and constant-asymmetry loss functions rise sharply. This advantage persists under heavy-tailed return distributions, and rankings by CZAR loss track directional skill rather than prediction variance.
We then evaluate CZAR in a LightGBM experiment on intraday (15-minute and 1-hour) BTC log-returns, where CZAR-trained models reduce the prediction shrinkage of L1- and L2-trained baselines and improve long--short performance and directional accuracy on large-magnitude returns (\S\ref{sec:evaluation}).
\S\ref{sec:discussion} discusses implications and limitations.

%%%%%%%%%%%%%%%%%%%%%%
\section{Background}
\label{sec:background}

\begin{figure}[t]
\centering
\includegraphics[width=0.99\textwidth]{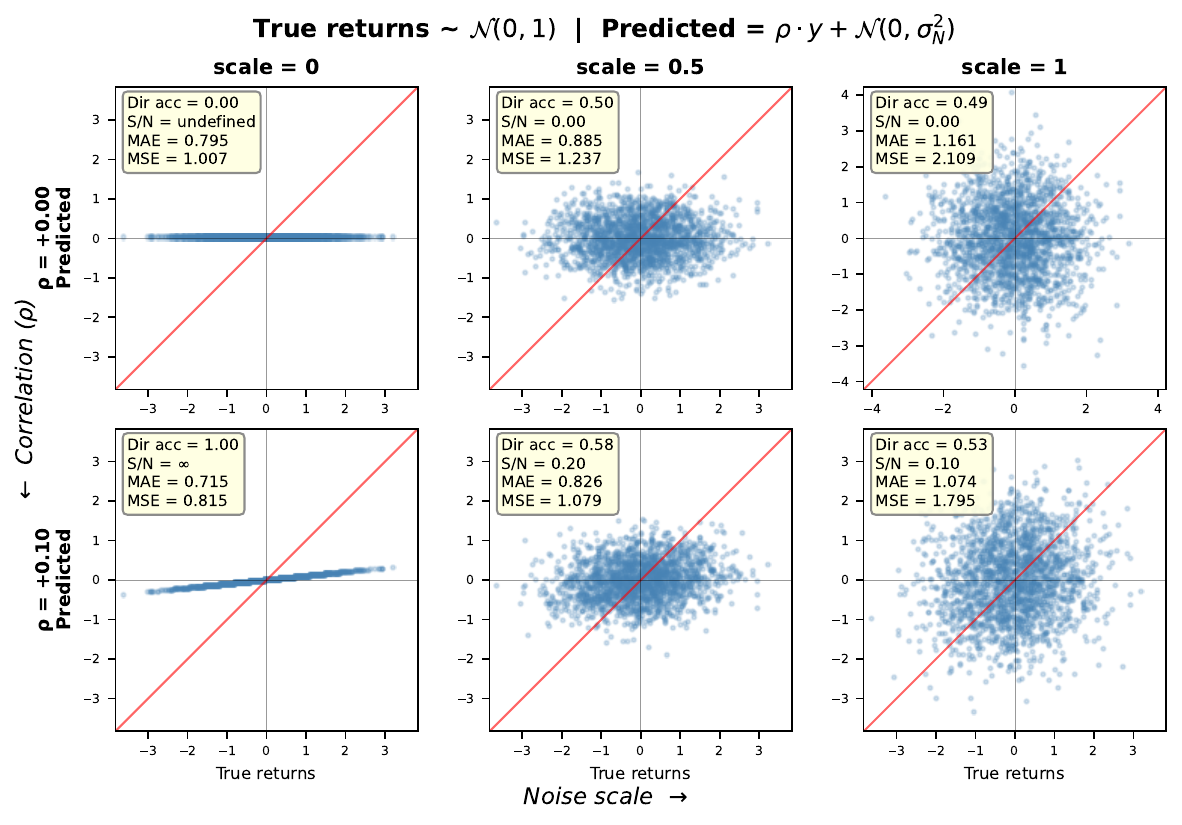}
\caption{Scatter plots of predicted vs.\ true returns under the Gaussian linear model $\hat{y} = \rho\,y + \mathcal{N}(0,\sigma_N^2)$ for signal coefficient $\rho \in \{0, 0.1\}$ (rows) and noise scale $\sigma_N \in \{0, 0.5, 1\}$ (columns). The red diagonal marks perfect prediction. Directional accuracy (DA), signal-to-noise ratio, mean absolute error, and mean squared error are annotated in each panel. The top-left panel ($\rho=0, \sigma_N=0$) is the zero-returns predictor with $\hat{y}\equiv 0$, collapsing all predictions to the horizontal axis. The remaining $\rho=0$ panels are pure-noise predictors $\hat{y}=\sigma_N\varepsilon$. Despite zero DA, the zero-returns predictor achieves lower MAE (0.795) and MSE (1.007) than the signal-carrying predictor at $\rho=0.1, \sigma_N=0.5$ (MAE~0.826, MSE~1.079), illustrating the zero-returns attractor problem.}
\label{fig:returns_model}
\end{figure}

As a starting point for evaluating loss functions for returns predictions, we consider a simple case where the underlying true returns follow a Gaussian distribution with mean at zero and scale of unity, $y = \mathcal{N}(\mu=0,\sigma=1)$.
Predicted returns then follow the format of Gaussian linear predictions, $\hat{y} = \rho y + \mathcal{N}(0, \sigma_N^2)$, where $\rho$ is the signal coefficient and $\sigma_\mathrm{N}$ is the noise scale, relative to the true distribution (i.e. for a fixed true return $y$, the prediction is drawn from a second Gaussian, whose additive noise is independent of $y$).
Note that $\rho$ is the regression slope, not the Pearson correlation between $\hat{y}$ and $y$; the latter is $\rho/\sqrt{\rho^2+\sigma_N^2}$.
Example distributions are shown in Figure~\ref{fig:returns_model} for increasing signal coefficient and noise scales.
In the following, we refer to the single case $\rho=0, \sigma_\mathrm{N}=0$ (so that $\hat{y}\equiv 0$) as the `zero-returns' model and the case $\rho=0, \sigma_\mathrm{N}=1$ (a pure-noise predictor whose marginal distribution matches the true returns PDF) as the `same-PDF' model. The remaining $\rho=0, \sigma_N>0$ panels in Figure~\ref{fig:returns_model} are pure-noise predictors, $\hat{y}=\sigma_N\varepsilon$, not zero-returns.
In Figure~\ref{fig:returns_model}, the directional accuracy, signal-to-noise ratio (which directly correlates with directional accuracy in this case) and mean loss values for MAE and MSE loss functions are indicated in each panel.
For $\rho=0$ with $\sigma_N>0$ (pure noise), directional accuracy is $\approx 50$\% (within some error due to sampling); for the zero-returns predictor ($\rho=0, \sigma_N=0$) every prediction is exactly zero and is reported as a directional miss, giving DA $=0$ in the top-left panel. For $\rho > 0$ directional accuracy decreases with increasing $\sigma_\mathrm{N}$.
A key point to note is that both MAE and MSE have lower losses for the zero-returns model than for the signal-carrying model with $\rho=0.1$ and $\sigma_\mathrm{N} = 0.5$; i.e., noise in the predictions is more important than the signal itself.
When evaluating the performance of models (e.g. when training and optimizing machine learning models to predict returns), the zero-returns case therefore becomes an attractor solution.

\begin{figure}[t]
\centering
\includegraphics[width=0.46\textwidth]{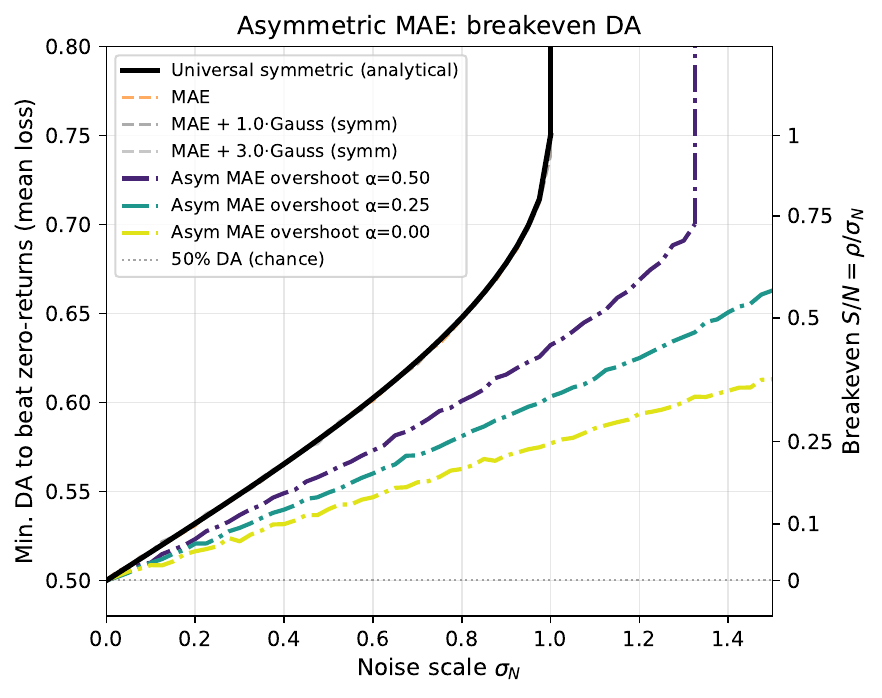}
\includegraphics[width=0.46\textwidth]{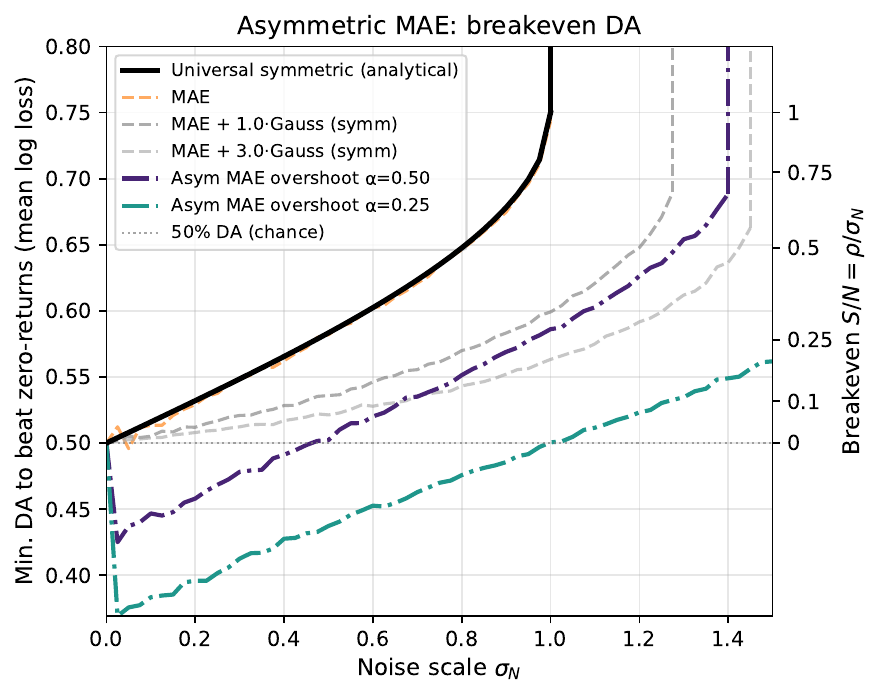}
\includegraphics[width=0.46\textwidth]{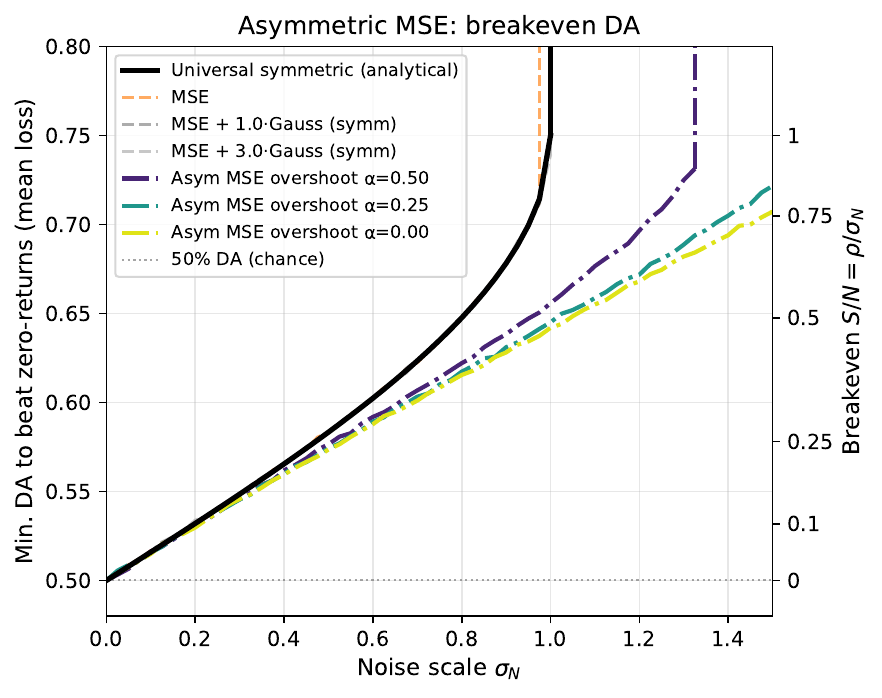}
\includegraphics[width=0.46\textwidth]{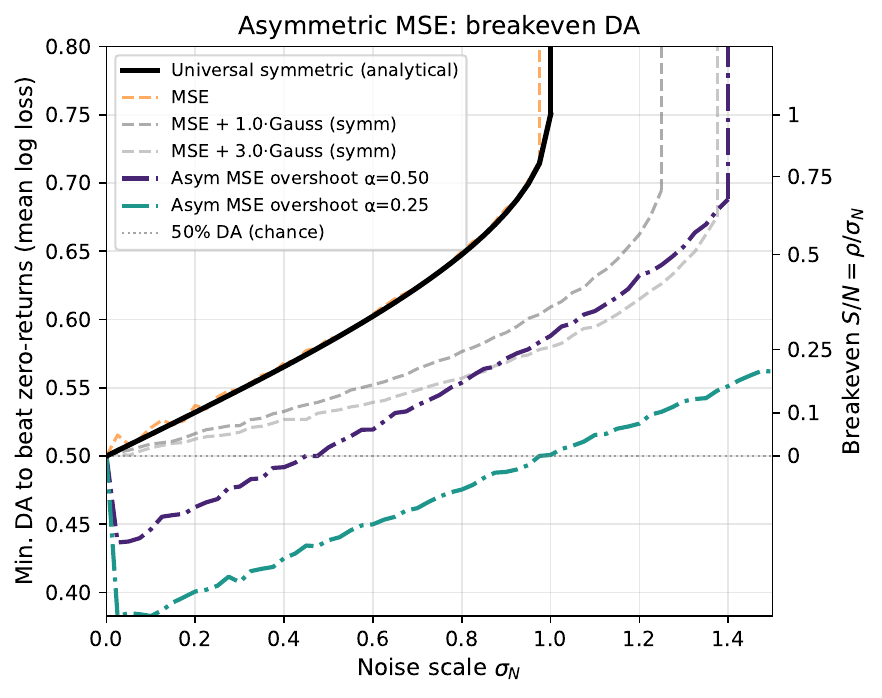}
\caption{Breakeven directional accuracy vs.\ noise scale $\sigma_N$ for asymmetric MAE (top row) and asymmetric MSE (bottom row), evaluated under mean linear loss (left) and mean log loss (right). Black solid lines show the universal symmetric analytical breakeven, orange dashed lines show symmetric MAE or MSE baseline, gray dashed lines show symmetric loss with an additive Gaussian floor ($\lambda \in \{1, 3\}$), and colored dash-dotted lines show asymmetric loss downweighting overshoots by factor $\alpha \in \{0, 0.25, 0.5\}$ (linear panels) or $\alpha \in \{0.25, 0.5\}$ (log panels). In all panels, the relations become vertical when  Gaussian linear models cannot achieve lower losses than the zero-returns model. \textnormal{Left panels:} For linear losses, asymmetry reduces the breakeven DA while remaining $>50\%$ for all $\sigma_N$. A linear undershooting term (MAE, $\alpha=0$, upper left panel) achieves far lower breakeven than a quadratic one (MSE, lower left panel), and the $\alpha$ sensitivity is negligible for MSE. \textnormal{Right panels:} For log losses, asymmetry pushes the breakeven below $50\%$ at small $\sigma_N$, meaning a model predicting consistently in the wrong direction could achieve a lower mean log loss than the zero-returns predictor, a clearly undesirable property. The Gaussian floor instead keeps the breakeven $\geq 50\%$ while reducing it at intermediate $\sigma_N$.}
\label{fig:breakeven_DA}
\end{figure}

This series of models can be used to calculate a `breakeven' directional accuracy, which occurs when the expectation values are identical, i.e. $\mathbb{E}\!\left[\mathcal{L}(y,\hat{y})\right] = \mathbb{E}\!\left[\mathcal{L}(y,0)\right]$.
That is, for a given parameter (e.g. noise scale) what is the directional accuracy needed to achieve the same loss as the zero-returns model.
We show this for the MAE and MSE loss functions in Figure~\ref{fig:breakeven_DA}, with breakeven directional accuracy as a function of noise scale.
Critically, for a Gaussian true returns distribution, monotonic loss functions and Gaussian linear prediction models, symmetric loss functions follow a universal breakeven function, $\mathrm{DA} = 0.5 + \arctan(\rho_\mathrm{min} \sigma / \sigma_\mathrm{N}) / \pi$ with $\rho_\mathrm{min} = 1 - \sqrt{1 - (\sigma_\mathrm{N}/\sigma)^2}$, valid for $0 \le \sigma_\mathrm{N} \le 1$ (both for evaluation with mean losses and the mean logarithm of the losses).
We derive this result in Appendix~\ref{app:breakeven}.
At a noise scale $\sigma_\mathrm{N} > 1$ (with $\sigma=1$), Gaussian linear models cannot achieve a lower loss than the zero-returns model, regardless of their directional accuracy.

What can improve the breakeven point is using asymmetric functions, which reward overestimates relative to underestimates.
This is shown by the dash-dotted colored lines in Figure~\ref{fig:breakeven_DA}, which use a constant $\alpha$ term to downweight losses for overestimates (analogous to quantile or LinLin loss functions).
For the MSE loss we set $\alpha^2$ scaling for equivalency with the MAE loss.
For evaluation with mean (linear) losses (left panels in the figure), smaller $\alpha$ lowers the gradient in the $\sigma_\mathrm{N}$-breakeven plane, but the breakeven points always remain $>50$\%.
The minimal breakeven point of course occurs when $\alpha=0$, i.e. there is no punishment for overshoots in the correct direction.
What is clear from comparing the MAE and MSE panels is that, for evaluation of mean losses, a linear term on the undershooting side is preferred over a quadratic term, since MAE with $\alpha=0$ achieves far lower breakeven directional accuracy than MSE.

Taking the mean log loss\footnote{Throughout, ``mean log loss'' (and ``log loss'' more generally) denotes log-averaged evaluation, i.e. the mean of the logarithm of the per-sample losses or equivalently the logarithm of their geometric mean. This is unrelated to the cross-entropy loss of classification, which is also commonly called `log loss'.} values for evaluation creates concave loss functions, which changes how asymmetric functions affect breakeven directional accuracy.
For evaluation with mean log losses, the asymmetry term $\alpha$ acts as a vertical scaling.
In this case, for small noise scales the breakeven point can occur at $<50$\%; i.e. models could be optimized to \textit{incorrectly} predict returns, which obviously should be avoided.
For mean log losses, reduction in breakeven directional accuracy can be achieved by adding a `floor' term to the loss function that peaks at $y = 0$ and decreases with increasing $|y|$.
We show this as the gray dashed lines in Figure~\ref{fig:breakeven_DA}, which adds a floor following a scaled Gaussian, $\lambda \exp(-y^2)$, to the symmetric MAE loss.
This achieves a similar reduction in breakeven directional accuracy for mean log losses as asymmetry does for mean losses.
The floor term has no effect on mean loss since it is independent of the predicted values and cancels from the breakeven inequality.

These tests, together with the demands of gradient-based training (a custom objective in gradient-boosted libraries must supply a well-behaved gradient and Hessian, which favors convexity in the prediction) and the need to adequately penalize extreme outliers, motivate five requirements for a loss function designed for return prediction. Such a loss should:
\begin{enumerate}[label=(\roman*)]
    \item be convex in the prediction $\hat{y}$ at fixed true value $y$ (joint convexity is not required), making it suitable as a training objective for gradient-boosted models;
    \item be asymmetric, so that correct-direction overshoots are punished less than undershoots or wrong-direction predictions for the same fixed error, $|\hat{y} - y|$, while becoming symmetric when the true return is zero ($y=0$);
    \item have a near-linear loss term for undershoots and wrong-direction predictions;
    \item diverge as the prediction error grows, so that $\mathcal{L}\to\infty$ for large errors and extreme outliers are adequately penalized; and
    \item have an adaptive loss floor that raises the loss for samples whose true returns lie near the distribution mean, thereby reducing the relative weight of these samples under log-averaged evaluation.
\end{enumerate}

%%%%%%%%%%%%%%%%%%%%%%%%%%%%
\section{CZAR loss function}
\label{sec:czar}

The CZAR (Composite Zero-Agnostic Return) loss function is convex in the prediction $\hat{y}$ at fixed true value $y$ (piecewise quadratic on either side of a region boundary at $\hat{z}=z$), designed around a single asymmetry: a prediction that overshoots the true return in the correct direction should be penalized less than one that undershoots or points in the wrong direction.
This section presents the core formulation.
Appendix~\ref{app:smoothing} describes an optional pseudo-Huber smoothing extension for undershooting (Region A); in preliminary experiments it did not improve performance, and it is not used in any result reported in this paper.

\subsection{Notation} \label{sec:czar:notation}

Given a true return $y$ and predicted return $\hat{y}$ with known distributional mean $\mu$ and standard deviation $\sigma$\footnote{$\mu$ and $\sigma$ are not fitted parameters of the loss but externally supplied properties of the target (true-return) distribution, treated as known constants for each sample. In practice they are estimated from recent history of the realized returns and the same values are used during training and evaluation. In our experiments (\S\ref{sec:eval:training}) $\sigma$ is a $100$-candle rolling standard deviation of realized returns and $\mu=0$, the natural choice for log-returns at short horizons where the mean is negligible relative to the volatility. However, we retain $\mu$ in the notation for generality.}, define standardized $z$-score variables
\begin{equation} \label{eq:standardize}
    z = \frac{y-\mu}{\sigma}, \qquad \hat{z} = \frac{\hat{y}-\mu}{\sigma},
\end{equation}
and let $\Delta z = |\hat{z} - z|$ denote the unsigned prediction error in standardized units. The key adaptive quantity is the asymmetry factor
\begin{equation} \label{eq:beff}
    \beta_{\mathrm{eff}}(|z|) = \frac{1}{1+\beta |z|} \in (0,\,1],
\end{equation}
which equals $1$ at $z=0$ and decays to $0$ as $|z| \to \infty$, parameterized by $\beta \ge 0$.
In the base loss below it multiplies the quadratic overshoot penalty, while its complement $1-\beta_{\mathrm{eff}}$ weights the linear undershoot term.
Its two limits implement requirements (ii)–(iii) of \S\ref{sec:background}.
At $z=0$ the loss is exactly symmetric, and as $|z|$ grows the asymmetry strengthens, so overshooting a large true move in the correct direction becomes nearly free while undershooting it is penalized at full linear weight.

\subsection{Base loss} \label{sec:czar:base}

The base loss splits into two regions determined by whether the prediction overshoots the true value:
\begin{enumerate}[label=\textbf{Region \Alph*}, labelindent=\parindent, leftmargin=*, align=left, widest*=2]
    \item ($\operatorname{sign}(\hat{z}) \neq \operatorname{sign}(z)$ or $|\hat{z}| \le |z|$): prediction does not overshoot; either wrong direction or correct direction but undershooting.
    \item ($\operatorname{sign}(\hat{z}) = \operatorname{sign}(z)$ and $|\hat{z}| > |z|$): prediction overshoots; correct direction with magnitude exceeding $|z|$.
\end{enumerate}
\begin{equation}
    \mathcal{L}_{\mathrm{base}}(z, \hat{z}) =
    \begin{cases}
        (1-\beta_{\mathrm{eff}})\,\Delta z + \dfrac{\alpha}{2}\,\Delta z^2
            & \text{(Region A)}, \\[10pt]
        \beta_{\mathrm{eff}}\,\dfrac{\alpha}{2}\,\Delta z^2
            & \text{(Region B)},
    \end{cases}
    \label{eq:czar-base}
\end{equation}
where $\alpha > 0$ is the quadratic penalty scale, applied uniformly across both regions.
This base loss satisfies the first four requirements from \S\ref{sec:background}.
At $z = 0$, $\beta_{\mathrm{eff}} = 1$ and both regions collapse to the symmetric quadratic $\mathcal{L}_{\mathrm{base}}\big|_{z=0} = \tfrac{\alpha}{2}\,\Delta z^2$ (with $\operatorname{sign}(0)=0$, every prediction formally falls in Region~A, but the distinction is immaterial since the two expressions coincide); as $|z|$ grows, $\beta_{\mathrm{eff}}$ decays and the asymmetry between regions increases.
The quadratic terms in both regions guarantee convexity in $\hat{y}$ (at fixed $y$) and unbounded growth for large prediction errors.
We accordingly restrict the CZAR domain to $\alpha > 0$.
At $\alpha = 0$ Region~B is identically zero (no penalty for overshoot in the correct direction), and Region~A reduces to a pure MAE-like $(1-\beta_{\mathrm{eff}})\,\Delta z$, so the unbounded-growth requirement from \S\ref{sec:background} fails on the overshoot side and the loss admits a one-parameter degenerate family of zero-loss predictions.
The correlated defaults compound the degeneracy because $\beta^*(0)=0$ and $C^*(0)=0$ collapse both the asymmetry and the loss floor, leaving the base loss flat in $\hat{y}$.
We therefore treat $\alpha\to 0$ only as a limiting or diagnostic regime (e.g.\ the small-$\alpha$ entries in Figure~\ref{fig:czar-grad-panel} and Table~\ref{tab:model-training} below stay strictly positive at $\alpha\ge 10^{-3}$), not as a CZAR operating point.

\begin{figure}[t]
    \centering
    \includegraphics[width=0.99\textwidth]{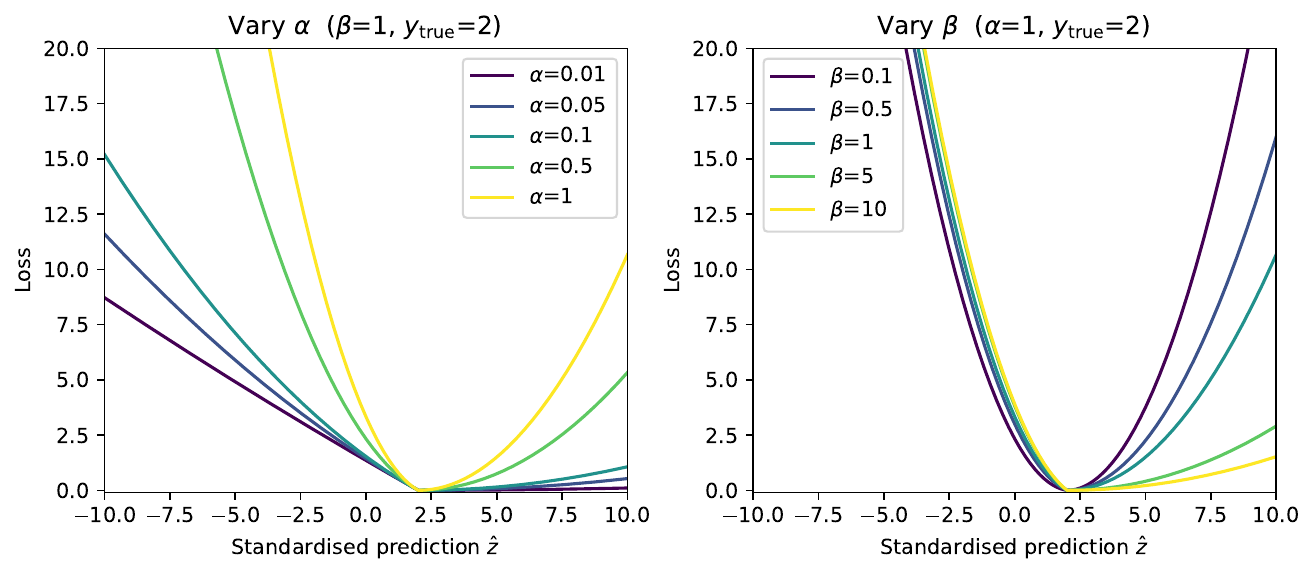}
    \caption{$\mathcal{L}_{\mathrm{CZAR}}$ as a function of standardized prediction~$\hat{z}$ for fixed true value $z=2$, showing sensitivity to the two primary hyperparameters. The left panel shows varying $\alpha$ at fixed $\beta=1$. Larger $\alpha$ steepens the quadratic penalty in both regions. The right panel shows varying $\beta$ at fixed $\alpha=1$. Larger $\beta$ suppresses Region~B (overshoot, $\hat{z}>z$), making the asymmetry more pronounced.}
    \label{fig:czar-params}
    \end{figure}
    
Figure~\ref{fig:czar-params} shows the sensitivity of the loss profile for a fixed true value $z=2$.
Here, $\alpha$ controls the overall quadratic scale uniformly across both regions, while $\beta$ governs the asymmetry by setting how quickly $\beta_{\mathrm{eff}}$ decays with $|z|$.

\subsection{Gradient and Hessian} \label{sec:czar:derivatives}

Let $s_\Delta = \operatorname{sign}(\hat{z}-z)$, set to $0$ at $\hat{z}=z$. Differentiating Equation~\eqref{eq:czar-base} with respect to $\hat{y} = \mu + \sigma\hat{z}$:
\begin{equation}
    \frac{\partial \mathcal{L}_{\mathrm{base}}}{\partial \hat{y}}
    = \frac{s_\Delta}{\sigma}
    \begin{cases}
        (1-\beta_{\mathrm{eff}}) + \alpha\,\Delta z
            & \text{(Region A)}, \\[6pt]
        \beta_{\mathrm{eff}}\,\alpha\,\Delta z
            & \text{(Region B)}.
    \end{cases}
    \label{eq:czar-grad}
\end{equation}
The second derivative is
\begin{equation}
    \frac{\partial^2 \mathcal{L}_{\mathrm{base}}}{\partial \hat{y}^2}
    = \frac{1}{\sigma^2}
    \begin{cases}
        \alpha
            & \text{(Region A)}, \\[6pt]
        \beta_{\mathrm{eff}}\,\alpha
            & \text{(Region B)}.
    \end{cases}
    \label{eq:czar-hess}
\end{equation}
The MAE term in Region~A contributes no curvature, so the Region~A Hessian is simply $\alpha/\sigma^2$.
At $\hat{z}=z$ with $z\neq 0$ the gradient has a step discontinuity of magnitude $(1-\beta_{\mathrm{eff}})/\sigma$.
The Region~A side-limit is $-\operatorname{sign}(z)\,(1-\beta_{\mathrm{eff}})/\sigma$ and the Region~B side-limit is $0$.
For $z>0$, Region~A lies left of the region boundary, so the gradient jumps from $-(1-\beta_{\mathrm{eff}})/\sigma$ up to $0$ as $\hat{z}$ increases through $z$.
For $z<0$ the sides reverse and the gradient jumps from $0$ up to $+(1-\beta_{\mathrm{eff}})/\sigma$.
In both signs the one-sided jump $\partial^+\mathcal{L}/\partial\hat{y} - \partial^-\mathcal{L}/\partial\hat{y} = (1-\beta_{\mathrm{eff}})/\sigma \ge 0$ is upward.
This concentrates the reward signal near the true value.
The loss pulls strongly toward $z$ from the Region~A side, while overshoot is penalized only quadratically.

For each fixed $z$, $\mathcal{L}_{\mathrm{base}}(z,\cdot)$ is convex in $\hat{y}$ (equivalently in $\hat{z}$).
On each open side of the region boundary the derivative $\partial\mathcal{L}_{\mathrm{base}}/\partial\hat{y}$ is affine in $\hat{z}$ with non-negative slope (the Hessian is $\alpha/\sigma^2$ on Region~A and $\beta_{\mathrm{eff}}\alpha/\sigma^2$ on Region~B, both $\ge 0$), and at $\hat{z}=z$ the left and right one-sided derivatives satisfy $\partial^+\mathcal{L}/\partial\hat{y} \ge \partial^-\mathcal{L}/\partial\hat{y}$ (upward jump $(1-\beta_{\mathrm{eff}})/\sigma\ge 0$ established above).
The subdifferential $\partial\mathcal{L}_{\mathrm{base}}(z,\cdot)$ is therefore monotone non-decreasing in $\hat{y}$, which is sufficient for convexity in $\hat{y}$ at fixed $y$. 
We also note that joint convexity in $(y,\hat{y})$ is not claimed and does not hold, since $\beta_{\mathrm{eff}}$ depends on $|z|$.

\begin{figure}[t]
    \centering
    \includegraphics[width=0.48\textwidth]{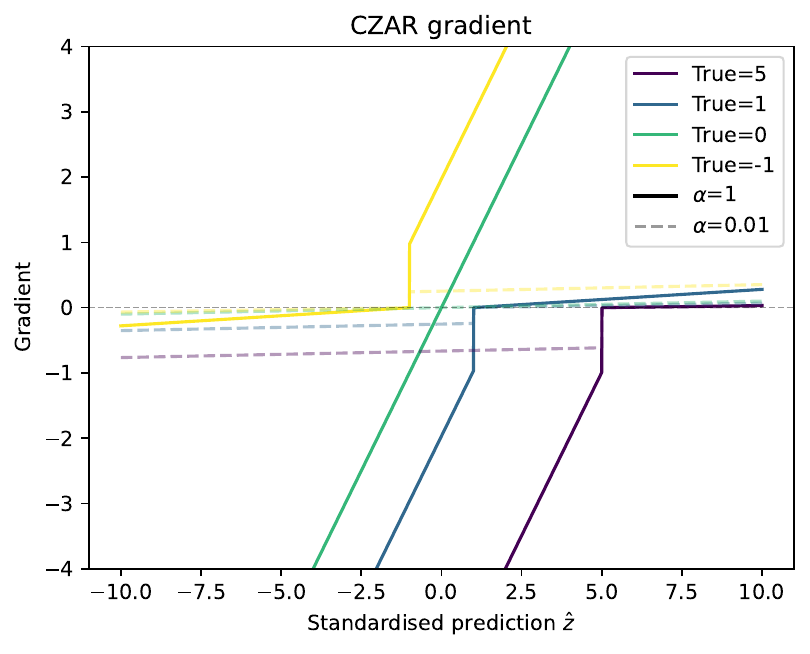}
    \hfill
    \includegraphics[width=0.48\textwidth]{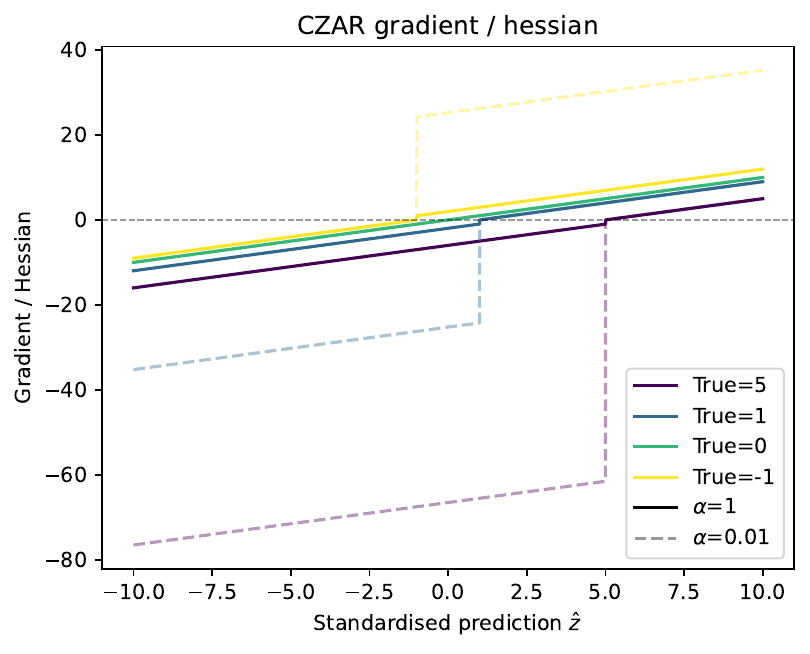}
    \caption{Gradient and Newton step of $\mathcal{L}_{\mathrm{CZAR}}$ for standardized true values $z \in \{-1, 0, 1, 5\}$. The left panel shows the gradient $\sigma\,\partial\mathcal{L}_{\mathrm{CZAR}}/\partial\hat{y}$ vs.\ standardized prediction~$\hat{z}$. Each curve has a step discontinuity at $\hat{z}=z$ of magnitude $1-\beta_{\mathrm{eff}}(|z|)$. The gradient approaches $-\operatorname{sign}(z)\,(1-\beta_{\mathrm{eff}})$ from Region~A and $0$ from Region~B, so it always jumps upward as $\hat{z}$ increases through $z$. For $z>0$ (e.g.\ $z=1,5$) Region~A lies to the left of the region boundary and the gradient jumps from $-(1-\beta_{\mathrm{eff}})$ up to $0$. For $z<0$ (e.g.\ $z=-1$) Region~A lies to the right and the jump runs from $0$ up to $+(1-\beta_{\mathrm{eff}})$. For large $|z|$ the step grows as $\beta_{\mathrm{eff}}\to 0$, while for $z=0$ ($\beta_{\mathrm{eff}}=1$) the gradient is continuous and linear (MSE-like), with no discontinuity. The right panel shows the Newton step $(\partial\mathcal{L}/\partial\hat{y})\,/\,(\partial^2\mathcal{L}/\partial\hat{y}^2)$ vs.\ $\hat{z}$. In Region~B the factor $\beta_{\mathrm{eff}}\alpha$ is shared by the gradient and Hessian and cancels in the ratio, leaving a Newton step of $\sigma,\Delta z = \hat{y}-y$ that is independent of $\alpha$ and $\beta_{\mathrm{eff}}$ (i.e. the overshoot damping seen in the loss and gradient does not carry over to the Newton step). In Region~A the constant Hessian ($\alpha/\sigma^2$) gives a Newton step $\sigma,\Delta z + \sigma(1-\beta_{\mathrm{eff}})/\alpha$, i.e. the MSE step plus a constant offset from the MAE term that grows as $|z|$ increases and as $\alpha$ shrinks. In both panels, solid lines show $\alpha=1$ and dashed lines show $\alpha=0.01$. Both use the correlated asymmetry term $\beta=\beta^*(\alpha)$ from Equation~\eqref{eq:correlated-beta} and floor level $C=C^*(\alpha)$ from Equation~\eqref{eq:correlated-C} (with $\tau=0.5$ and $\mu=0$).}
    \label{fig:czar-grad-panel}
\end{figure}
    
The gradient and Newton step are shown in Figure~\ref{fig:czar-grad-panel}.
The step discontinuity is visible for all $z\neq 0$, while at $z=0$, $\beta_{\mathrm{eff}}=1$ and the gradient is continuous and linear throughout.
At fixed $\beta$, varying $\alpha$ would change the post-step slope while leaving the step height unchanged, since the step arises from the MAE term rather than the quadratic.
Under the correlated default $\beta^*(\alpha)$, however, smaller $\alpha$ also reduces $\beta$, so the step height $1-\beta_{\mathrm{eff}}(|z|)$ shrinks alongside the slope, visible as the much smaller steps for the dashed ($\alpha=0.01$, $\beta\approx 0.32$) curves compared with the solid ($\alpha=1$, $\beta\approx 31.2$) curves.
The Newton step (right panel) behaves differently in the two regions.
In Region~B the shared factor $\beta_{\mathrm{eff}}\alpha$ cancels between gradient and Hessian, so the step is exactly $\sigma,\Delta z = \hat{y}-y$, independent of $\alpha$ and $\beta_{\mathrm{eff}}$; i.e. the overshoot damping present in the gradient leaves the Newton step unchanged.
In Region~A the same $\sigma,\Delta z$ is augmented by a constant $\sigma(1-\beta_{\mathrm{eff}})/\alpha$ from the MAE term, which grows as $|z|$ increases and $\alpha$ decreases, so the undershoot side takes the largest Newton steps for large true values.

\subsection{Adaptive loss floor} \label{sec:czar:floor}

The base loss of a zero prediction approaches zero as $|z|\to 0$, so samples with true returns near zero contribute very little to the loss, allowing a zero-returns predictor to achieve a deceptively low mean log loss.
Therefore, to satisfy requirement~(v) of \S\ref{sec:background}, we add a floor that raises the minimum loss at small true returns.
The floor equals~$C$ at $z=0$ and decays to zero as $|z|$ grows.

Specifically, writing the zero-prediction base loss as
\begin{equation} \label{eq:L0}
\begin{split}
    \mathcal{L}_0(z) &= \mathcal{L}_{\mathrm{base}}(z,\,0) \\
                     &= (1-\beta_{\mathrm{eff}})\,|z| + \frac{\alpha}{2}\,z^2
\end{split}
\end{equation}
(a zero prediction never overshoots, so always falls in Region~A with $\Delta z=|z|$), the increase required to raise $\mathcal{L}_0(z)$ to at least $C$ is $\max\bigl(C - \mathcal{L}_0(z),\,0\bigr)$.
To avoid unintended behavior (see Appendix~\ref{app:optimization}), we implement a smoothed version of the rectification function.
Define the smooth hinge
\begin{equation} \label{eq:hinge}
    \mathrm{hinge}(x,\,\tau) = \tfrac{1}{2}\!\left(x + \sqrt{x^2 + \tau^2}\right),
\end{equation}
which approximates $\max(x,0)$ for small $\tau > 0$.
The loss floor then becomes
\begin{equation} \label{eq:lfloor}
    \mathcal{L}_{\mathrm{floor}}(z) = C\cdot\frac{\mathrm{hinge}(C - \mathcal{L}_0(z),\,\tau)}{\mathrm{hinge}(C,\,\tau)}.
\end{equation}
The denominator removes the upward bias of the hinge, so the floor equals exactly~$C$ at $z=0$ (where $\mathcal{L} _0=0$) and approaches zero smoothly as $\mathcal{L}_0$ grows beyond~$C$.
The total loss at $\hat{z}=0$ is thus pinned to $\approx\max(\mathcal{L}_0,\,C)$, non-decreasing in $|z|$ because $\mathcal{L}_0$ is strictly increasing.

Since $\mathcal{L}_{\mathrm{floor}}$ does not depend on $\hat{z}$, it contributes nothing to the gradient or Hessian and is invisible to gradient-based training.
The pressure away from the zero attractor in model training is thus carried entirely by the Region~A asymmetry of $\mathcal{L}_{\mathrm{base}}$ (for any $z\neq 0$, Equation~\eqref{eq:czar-grad} gives a gradient signed away from zero).
The role of the loss floor is instead in evaluation, ranking, and model selection.
It prevents a zero-returns predictor from achieving a misleadingly low mean log loss and, with the correlated default $C^*(\alpha)$, keeps the breakeven directional accuracy at or above~$50\%$ across noise scales (Appendix~\ref{app:optimization}).

\subsection{Total loss}\label{sec:czar:total}

\begin{table}[t]
    \centering
    \begin{tabular}{clcc}
    \hline
    Parameter & Role & Domain & Default \\
    \hline
    $\alpha$ & Quadratic regularization strength & $\alpha > 0$ & $1$ \\
    $\beta$  & Asymmetry decay rate              & $\beta \ge 0$  & $\beta^*(\alpha)$, eq.~\eqref{eq:correlated-beta}    \\
    $C$      & Floor target level                & $C \ge 0$        & $C^*(\alpha)$, eq.~\eqref{eq:correlated-C}    \\
    $\tau$   & Hinge smoothing width             & $\tau > 0$     & $0.5$    \\
    \hline
    \end{tabular}
    \caption{CZAR hyperparameters. At $\beta=0$, $\beta_{\mathrm{eff}}=1$ for all $z$ and both regions are purely quadratic (MSE-like). Larger $\beta$ increases the asymmetry between Region~A (MAE+quadratic) and Region~B (discounted quadratic) for returns with large $|z|$.}
    \label{tab:params}
    \end{table}
    
The total CZAR loss therefore becomes
\begin{equation}
    \mathcal{L}_{\mathrm{CZAR}}(y, \hat{y})
    = \mathcal{L}_{\mathrm{base}}(z, \hat{z}) + \mathcal{L}_{\mathrm{floor}}(z),
    \label{eq:czar-total}
\end{equation}
where the raw pair $(y, \hat{y})$ is converted to $(z, \hat{z})$ through Equation~\eqref{eq:standardize} and the properties of the target distribution ($\mu$, $\sigma$).
An example Python implementation for the loss function is described in Appendix~\ref{app:code}.

The loss function has four hyperparameters ($\alpha$, $\beta$, $C$, $\tau$), which are summarized in Table~\ref{tab:params}.
However, the parameters $\alpha$, $\beta$ and $C$ are not independent in practice.
For each value of $\alpha$, there is an optimal $\beta^*(\alpha)$ that minimizes the breakeven directional accuracy (as defined in \S\ref{sec:background}).
Sweeping both parameters over the breakeven DA at $\sigma_N = 1$ yields optimal pairs well described by a two-term power-law
\begin{equation}
    \beta^*(\alpha) = 4.2\,\alpha^{0.56} + 27.0\,\alpha^{2.17} ,
    \label{eq:correlated-beta}
\end{equation}
on the calibrated range $10^{-4} \leq \alpha \leq 0.7$ (at larger $\alpha$ the breakeven basin in $\beta$ becomes too shallow to fit, see Appendix~\ref{app:optimization}).
At the reference value $\alpha=1$, Equation~\eqref{eq:correlated-beta} gives $\beta^*\approx 31.2$. The implementation uses $\beta^*(\alpha)$ as the default when $\beta$ is not specified explicitly.
We show the full optimization results in Appendix~\ref{app:optimization}.

Similarly, an optimal floor level $C^*(\alpha)$ can be found by minimizing the log-loss\footnote{Recall that constant terms in loss functions do not affect the breakeven directional accuracy for linear losses, see \S\ref{sec:background}.} breakeven directional accuracy at fixed $\tau$, subject to the constraint that it remains $\ge 50$\%. The optimal values are well-described by an asymmetric Hill-bell model with a linear correction
\begin{equation}
    C^*(\alpha) = A\, h(\alpha)^k\, [1-h(\alpha)] + (R + m\alpha)\, h(\alpha),
    \qquad h(\alpha) = \frac{\alpha^p}{t^p + \alpha^p},
    \label{eq:correlated-C}
\end{equation}
over the range $10^{-5} \leq \alpha \leq 1$ with $(A, t, p, k, R, m) = (7.90,\, 4.59\times 10^{-3},\, 0.657,\, 0.684,\, 2.25,\, -0.218)$.
This form satisfies $C^*(0)=0$, peaks near $\alpha\sim 10^{-2}$, and decays slowly at larger $\alpha$.
The implementation uses $C^*(\alpha)$ as the default when $C$ is not specified explicitly; at $\alpha=1$ this gives $C^*\approx 2.19$.
The hinge width $\tau$ is relatively insensitive to $\alpha$ and $\beta$; we adopt $\tau=0.5$ as the default.
As with $\beta^*(\alpha)$, the fit in Equation~\eqref{eq:correlated-C} is intended for the working range $\alpha \lesssim 1$ (Appendix~\ref{app:optimization}).
Extrapolating well beyond that range, the linear correction causes $C^*(\alpha)$ to become negative near $\alpha \approx 10$, outside the domain $C \ge 0$ of Table~\ref{tab:params}.
The reference implementation (Appendix~\ref{app:code}) therefore disables the loss floor whenever $C \le 0$.

\begin{figure}[t]
    \centering
    \includegraphics[width=0.5\textwidth]{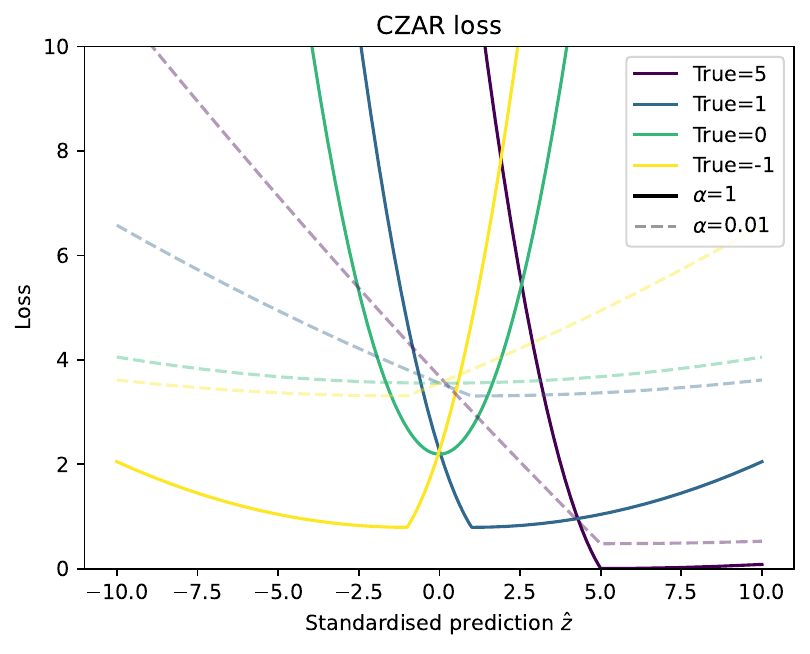}
    \caption{$\mathcal{L}_{\mathrm{CZAR}}$ as a function of standardized prediction~$\hat{z}$ for several standardized true values~$z$ (default parameters, $\mu=0$) and $\alpha = 1$ (solid lines) and $0.01$ (dashed lines). Each curve attains its minimum at $\hat{z}=z$. For large $|z|$, Region~A (undershooting or wrong direction) adopts a near-linear profile while Region~B (overshoot) remains quadratic, reflecting the asymmetry encoded by $\beta_{\mathrm{eff}}$. The loss floor raises the minimum loss for small $|z|$, visible as the elevated minimum of the $z=0$ curve.}
    \label{fig:czar-loss}
\end{figure}
    
Figure~\ref{fig:czar-loss} shows $\mathcal{L}_{\mathrm{CZAR}}$ as a function of standardized prediction~$\hat{z}$ for several true values, at two parameter sets each using the correlated default $\beta^*(\alpha)$.
Each curve is asymmetric about its minimum at $\hat{z}=z$.
For large $|z|$, Region~A (undershoot and wrong-direction predictions) is nearly linear, i.e. $\beta_{\mathrm{eff}}\to 0$ and the MAE-like term $(1-\beta_{\mathrm{eff}})\Delta z$ dominates.
Region~B (overshoot) is instead near-flat for large $|z|$.
The discounted quadratic $\beta_{\mathrm{eff}}\frac{\alpha}{2}\Delta z^2$ contributes negligible penalty once $\beta_{\mathrm{eff}}$ is small, so overshooting a large true return is almost free.
At $z=0$, the two regions collapse to a single MSE-like parabola; the loss floor then provides the bulk of the loss, pinning the zero-prediction loss to the target level~$C$.
Smaller $\alpha$ (dashed) scales down the overall loss without qualitatively changing the shape.

\begin{figure}[t]
    \centering
    \includegraphics[width=0.9\textwidth]{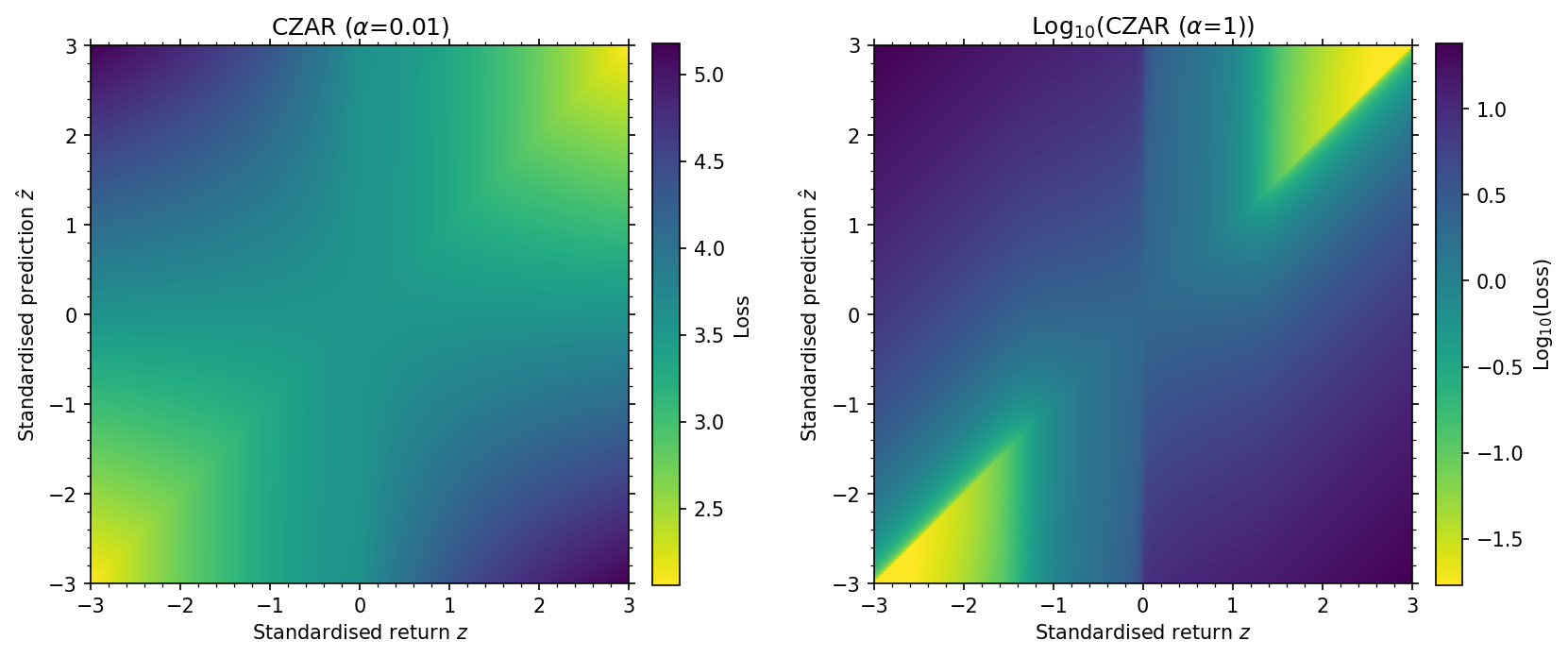}
    \includegraphics[width=0.9\textwidth]{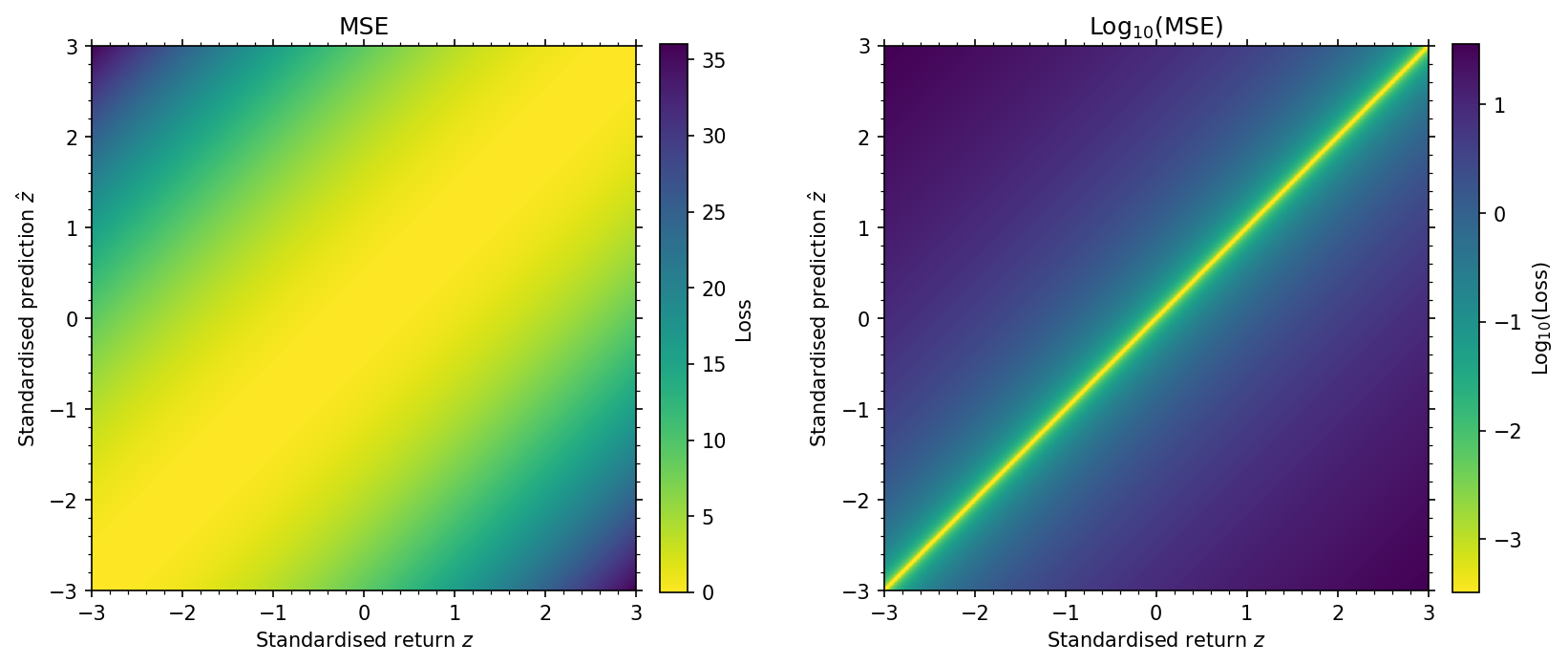}
    \caption{Loss in the $(z,\hat{z})$ plane for CZAR (top) and MSE (bottom). Left panels show loss on a linear scale, as seen by a gradient-based optimizer; right panels show $\log_{10}(\mathrm{loss})$, more representative of the dynamic range used when ranking models. The CZAR linear panel uses $\alpha=0.01$ (where the MAE-like undershoot regime is most visible on a linear scale) while the CZAR log panel uses $\alpha=1$ (where the dynamic range is best resolved on a log scale); both use the correlated defaults $\beta^*(\alpha)$, $C^*(\alpha)$, and $\tau=0.5$.}
    \label{fig:color-maps}
\end{figure}
    
To see the full shape of the loss surface rather than slices through it, Figure~\ref{fig:color-maps} plots $\mathcal{L}_{\mathrm{CZAR}}$ over the $(y,\hat{y})$ plane alongside MSE as a reference symmetric baseline.
MSE is symmetric about the diagonal $\hat{y}=y$ and depends only on the residual magnitude $(\hat{y}-y)^2$.
CZAR breaks this symmetry along the orthogonal direction.
Predictions that share the sign of the true return but overshoot in magnitude (off-diagonal points in the upper-right and lower-left quadrants) incur only a small additional penalty above the diagonal, whereas undershoots and wrong-direction predictions (off-diagonal points in the opposite-sign quadrants) are penalized much more heavily.
The loss floor seen as the elevated minimum of the $z=0$ curve in Figure~\ref{fig:czar-loss} appears here as a bright horizontal band of elevated loss along $y=0$ in the CZAR maps.
This raises the sample-averaged (ranking) loss of a near-zero predictor, since predictions in this strip cannot bring their loss arbitrarily close to zero; it does not act on per-sample gradients, which depend only on $\mathcal{L}_{\mathrm{base}}$ (\S\ref{sec:czar:floor}).
The corresponding training-time suppression of $\hat{z}\approx 0$ is therefore carried by the Region~A asymmetry of $\mathcal{L}_{\mathrm{base}}$, whose gradient is signed away from zero whenever $z\neq 0$.

%%%%%%%%%%%%%%%%%%%%%%
\section{Evaluation}
\label{sec:evaluation}

\subsection{Idealized tests}
\label{sec:eval:idealized}

\begin{figure}[t]
    \centering
    \includegraphics[width=0.6\textwidth]{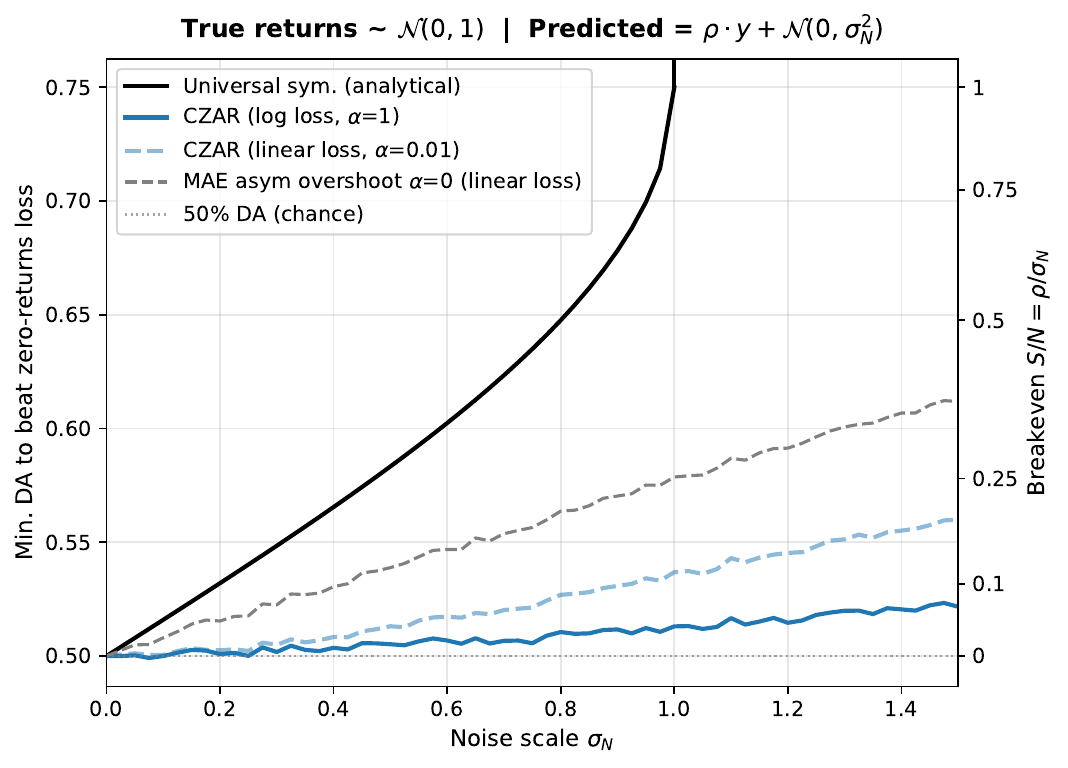}
    \caption{Breakeven directional accuracy vs.\ noise scale $\sigma_N$ under the Gaussian linear prediction model. The left axis shows the minimum DA required to beat the zero-returns predictor; the right axis shows the equivalent breakeven signal-to-noise $\rho/\sigma_N$. The black solid line shows universal analytical breakeven for symmetric monotonic losses, $\mathrm{DA}=0.5+\arctan(\rho_\mathrm{min}/\sigma_N)/\pi$ with $\rho_\mathrm{min}=1-\sqrt{1-\sigma_N^2}$; the gray dashed line shows the lower envelope achievable with the asymmetric MAE loss ($\alpha=0$) under mean linear loss (reproduced from Figure~\ref{fig:breakeven_DA}); the light blue dashed line shows CZAR under mean linear loss ($\alpha=0.01$); and the dark blue solid shows CZAR under mean log loss ($\alpha=1$). Both CZAR curves use the correlated defaults $\beta^*(\alpha)$, $C^*(\alpha)$, and $\tau=0.5$.}
    \label{fig:czar_breakeven}
\end{figure}
    
We revisit the Gaussian linear testbed of \S\ref{sec:background}, $y\sim\mathcal{N}(0,1)$ with $\hat{y} = \rho\,y + \sigma_N\,\varepsilon$, and ask two questions of CZAR: how does the breakeven directional accuracy compare to the symmetric and constant-asymmetry envelopes derived in \S\ref{sec:background}; and does the ordering induced by CZAR over a broad $(\rho,\sigma_N)$ sweep correctly track the prediction-quality metrics that matter downstream?
Throughout this subsection, CZAR uses the correlated defaults $\beta^*(\alpha)$ and $C^*(\alpha)$ from Equations~\eqref{eq:correlated-beta} and~\eqref{eq:correlated-C}, with $\tau=0.5$.
Loss expectations are estimated by Monte Carlo over $5\times 10^4$ samples per $(\rho,\sigma_N)$ pair.

Figure~\ref{fig:czar_breakeven} repeats the breakeven analysis of \S\ref{sec:background} for CZAR.
Under mean log loss, CZAR tracks the $50\%$ chance line to within $\lesssim 5$\,pp across the full $\sigma_N\in[0,1.5]$ range, a substantial improvement over both the universal symmetric envelope (which rises sharply as $\sigma_N\to 1$) and the lower bound achievable with a purely asymmetric MAE loss ($\alpha=0$ overshoot reward).
Under mean linear loss, CZAR's breakeven is intermediate but still improves on the asymmetric-MAE envelope at every noise scale.
In the Gaussian linear testbed, the combination of asymmetry and the adaptive loss floor substantially reduces the zero-returns breakeven directional-accuracy threshold, narrowing the regime in which a model with genuine but noisy directional skill ranks below the zero-returns predictor, a regime that persists even for the best-tuned symmetric or constant asymmetry losses in Figure~\ref{fig:breakeven_DA}.

\begin{figure}[t]
    \centering
    \includegraphics[width=\textwidth]{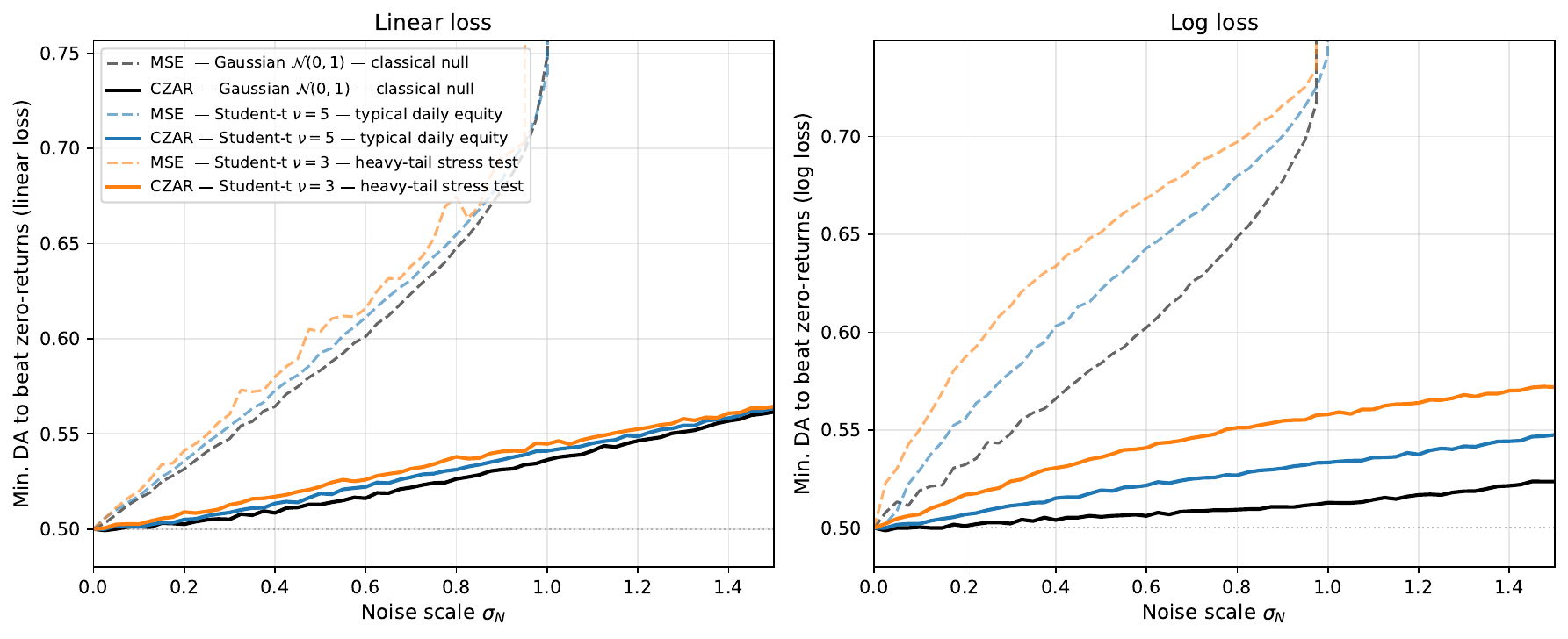}
    \caption{Breakeven directional accuracy vs.\ noise scale $\sigma_N$ under Gaussian (black lines) and heavy-tailed Student-$t$ truth distributions (blue lines for $\nu=5$ and orange lines for $\nu=3$), all normalized to unit variance. The left panel shows mean linear loss. The right panel shows mean log loss. In both panels, solid lines show the CZAR loss with the correlated defaults $\beta^*(\alpha)$, $C^*(\alpha)$, $\tau=0.5$ ($\alpha=0.01$ for linear loss, $\alpha=1$ for log loss), and dashed lines show the MSE loss. Under linear loss the breakeven curves are close to distribution-invariant. Under log loss heavier tails raise the breakeven DA for both losses, but CZAR retains its margin over MSE at every $\sigma_N$ and tail index.}
    \label{fig:czar_breakeven_heavy_tail}
\end{figure}
    
The Gaussian linear testbed is a deliberately friendly baseline; real return distributions are heavier-tailed, with empirically documented tail indices in the range $\nu\sim 3$--$5$ \citep{blattberg74,cont01,gopikrishnan99}.
Figure~\ref{fig:czar_breakeven_heavy_tail} repeats the breakeven sweep with $y$ and the noise term $\varepsilon$ for $\hat{y}$ drawn from unit-variance Student-$t$ distributions at $\nu=5$ (typical daily equity) and $\nu=3$ (heavy-tail stress test, for which the kurtosis is infinite).
Under linear loss, both CZAR and MSE breakeven curves are empirically close across these unit-variance distributions, and the CZAR advantage over MSE is preserved across all three.
Under log loss, heavier tails raise the breakeven DA for both losses, but CZAR retains a substantial margin over MSE at every $\sigma_N$ and at every tail index considered.
The mild degradation of CZAR's log-loss breakeven under heavier tails is at least partly an artifact of our defaults.
$C^*(\alpha)$ and $\tau$ were calibrated against the Gaussian linear testbed, so re-optimizing these parameters against a heavy-tailed reference distribution should recover much of the Gaussian-case performance for problems where heavy tails are expected.
The combination of asymmetry and adaptive loss floor therefore does not rely on Gaussian tails to deliver its advantage.

\begin{figure}[pt]
    \centering
    \includegraphics[width=0.95\textwidth]{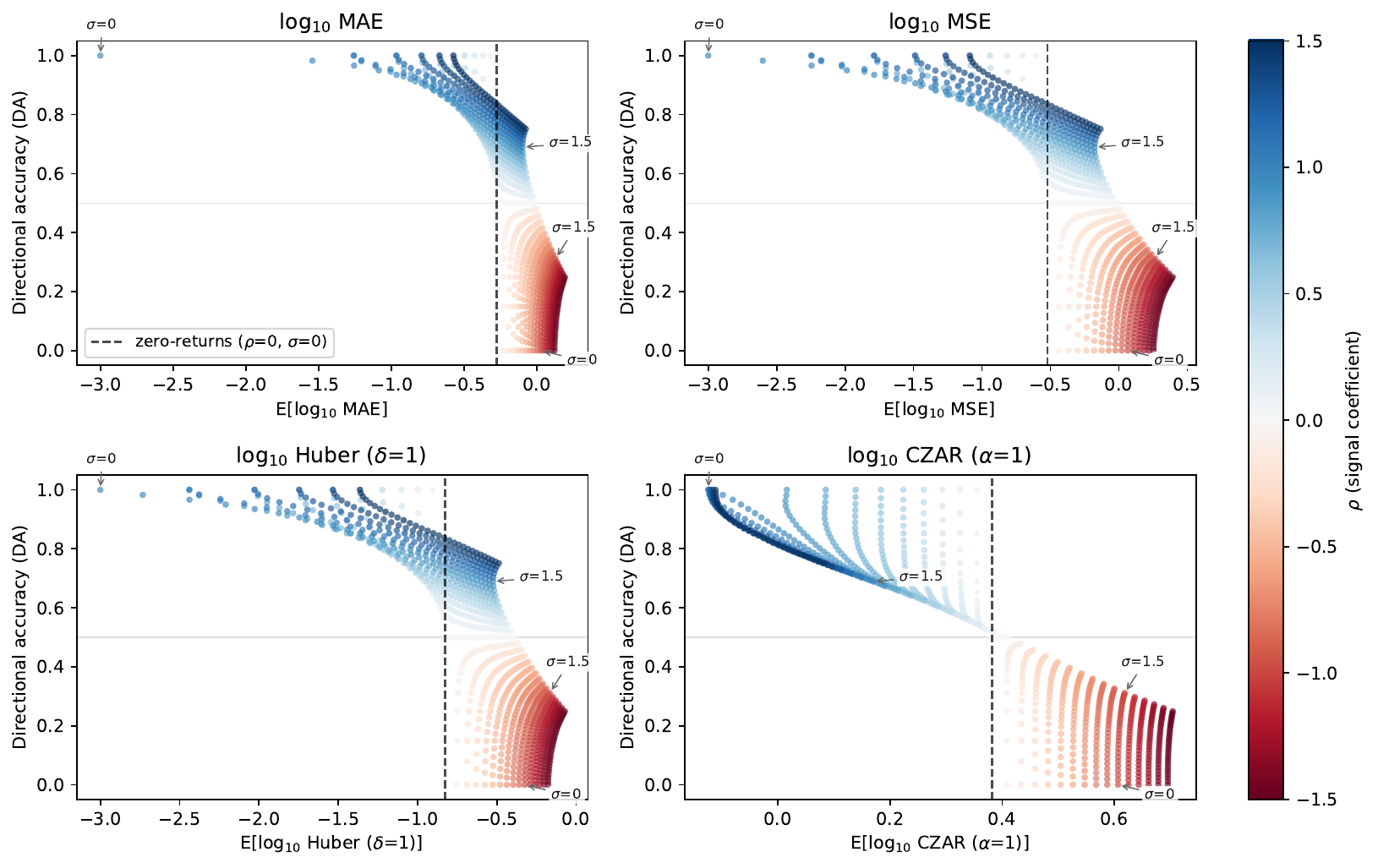}
    \includegraphics[width=0.95\textwidth]{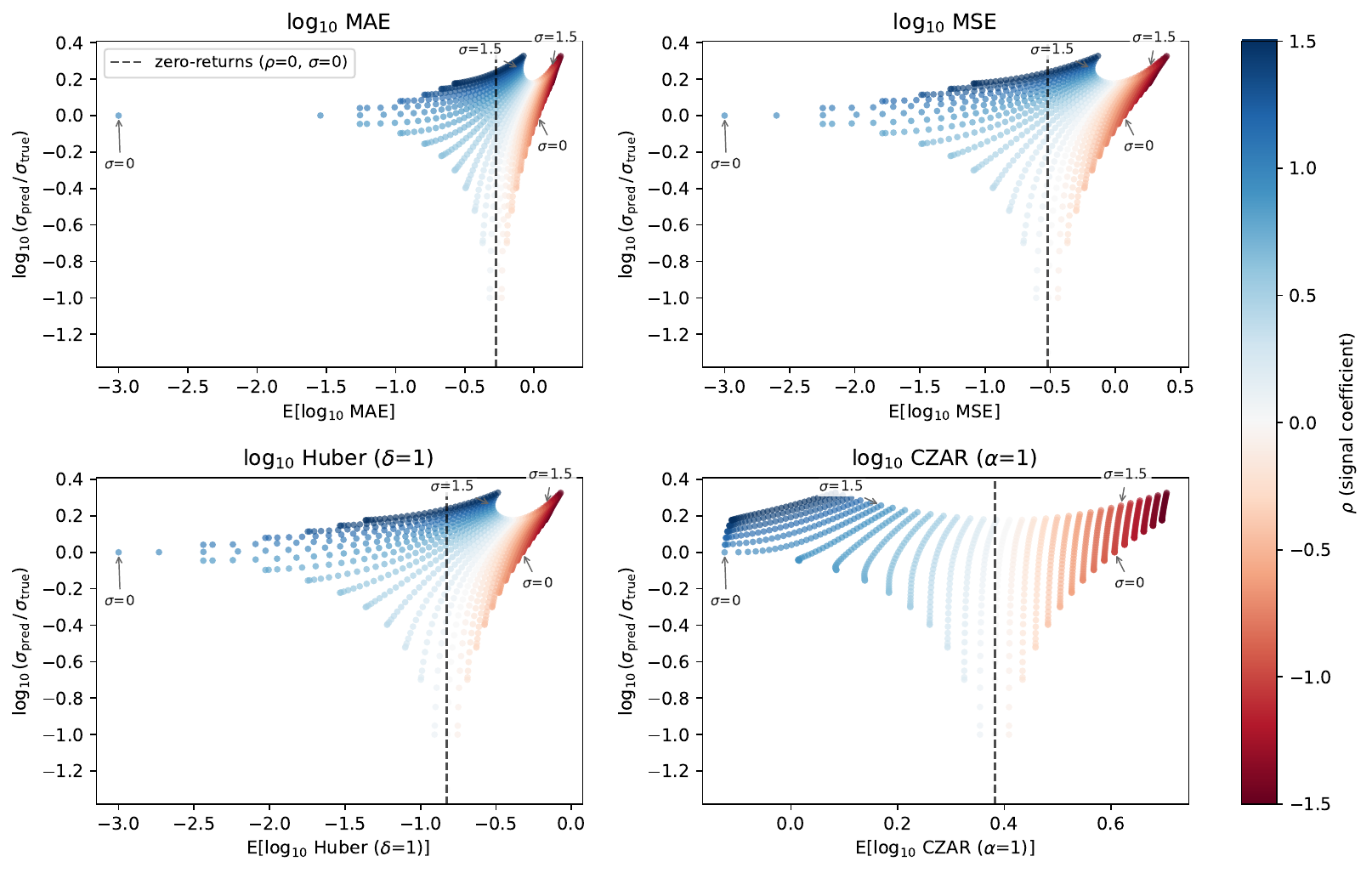}
    \caption{Sample-averaged $\log_{10}$-loss vs.\ prediction-quality metric for the Gaussian linear model, swept over $\rho\in[-1.5,1.5]$ and $\sigma_N\in[0,1.5]$. Each point is one $(\rho,\sigma_N)$ pair, with color encoding $\rho$ (blue for positive, red for negative); a small floor $\epsilon=10^{-3}$ is added before taking the log to avoid divergence at zero residuals. The floor is orders of magnitude below the zero-returns baseline of every loss shown and matters only for near-perfect predictors ($\rho\approx 1$, $\sigma_N\approx 0$), so the position of the dashed line and the orderings discussed in the text are unaffected by its value. The dashed vertical line in each panel marks the log-loss of the zero-returns predictor ($\rho=0$, $\sigma_N=0$); points to its left beat the zero-returns baseline. Columns compare MAE, MSE, Huber ($\delta=1$) and CZAR ($\alpha=1$ with the correlated defaults $\beta^*(\alpha)$, $C^*(\alpha)$, and $\tau=0.5$). \emph{Top:} directional accuracy (DA) on the vertical axis. \emph{Bottom:} log aspect ratio $\log_{10}(\sigma_\mathrm{pred}/\sigma_\mathrm{true})$ on the vertical axis.}
    \label{fig:metric-scatter}
\end{figure}
    
The breakeven analysis collapses each $(\rho,\sigma_N)$ pair to a single scalar (the minimum DA needed to beat zero returns).
Figure~\ref{fig:metric-scatter} unfolds this view by comparing the mean log-loss against two prediction-quality metrics (directional accuracy, top, and log aspect ratio $\log_{10}(\sigma_\mathrm{pred}/\sigma_\mathrm{true})$, bottom) across a wide sweep $\rho\in[-1.5,1.5]$, $\sigma_N\in[0,1.5]$.
The dashed vertical line in each panel is the log-loss of the zero-returns predictor, so models to its left rank above the zero-returns baseline.
Ideally, all models with positive signal should have lower losses than the zero-returns predictor.
Two failure modes of the symmetric losses (MAE, MSE, Huber) are visible.

\textit{Directional-accuracy failure mode (top panels).}
For each symmetric loss, the high-$|\rho|$, high-$\sigma_N$ branch (models with DA $\approx 0.6$--$0.7$) sits entirely to the right of the zero-returns line.
Any ranking based on these losses therefore prefers the trivial constant predictor to a noisy model with real directional skill.
Under CZAR, the asymmetric overshoot discount lowers the loss of directionally correct ($\rho>0$) noisy predictions, moving this branch leftward across the zero-returns line, while the negative-signal ($\rho<0$) branch remains in the penalized region; informative noisy models are then correctly preferred.
Beyond this shift in position, the $\rho>0$ branch also collapses onto a single DA--log-loss track.
Different $(\rho,\sigma_N)$ pairs with the same DA land at essentially the same CZAR loss, so ranking by mean CZAR loss tracks ranking by DA directly.
Symmetric losses fail this alignment because they spread the same DA values across a wide log-loss band, reflecting residual magnitude rather than directional correctness.

\textit{Aspect-ratio degeneracy (bottom panels).}
At fixed $\rho$, increasing $\sigma_N$ moves $\hat{y}=\rho y+\sigma_N\varepsilon$ to higher $\log_{10}\mathrm{AR}$ as the noise contribution to $\sigma_{\hat y}$ grows.
Under the symmetric losses, this sweep also raises the log-loss.
A fixed-$\rho$ sequence thus traces a diagonal in the panel, so noise is directly punished at fixed signal.
A consequence is that a model with no directional skill can reduce its loss simply by shrinking the magnitude of its predictions, since this reduces the average residual magnitude; the well-known regression-to-the-mean attractor.
Under CZAR, the same fixed-$\rho$ sequence runs nearly vertically (at least for $\rho\lesssim 0.5$).
The mean log-loss is essentially independent of $\sigma_N$ and is set by the signal level alone.
Ranking by mean CZAR loss therefore tracks the directional signal rather than the prediction's overall variance, decoupling rank from prediction noise and magnitude.

\subsection{Model training}
\label{sec:eval:training}

\begin{table}[t]
    \centering
    \small
    \begin{tabular}{lrrrrrrr}
    \hline
    Training loss & DA & DA$_{\mathrm{IQR}}$ & DA$_{1\sigma}$ & $\log_{10}$AR & $R_P$ & IC & Sharpe$_\mathrm{naive}$ \\
    \hline
    \multicolumn{8}{l}{\textit{BTC 15-minute}}\\
    \hline
    L1                    & 0.510 & 0.487 & 0.473 & $-0.93$   & 0.019  & 0.018  & $-4.97$ \\
    L2                    & 0.514 & 0.508 & 0.499 & $-1.18$   & 0.016  & 0.019  & $0.96$  \\
    CZAR ($\alpha=0.005$) & 0.531 & 0.523 & 0.537 & $-0.46$   & 0.046  & 0.034  & $5.62$  \\
    CZAR ($\alpha=0.01$)  & 0.522 & 0.504 & 0.501 & $-0.44$   & 0.017  & 0.033  & $-0.95$ \\
    CZAR ($\alpha=0.05$)  & 0.523 & 0.525 & 0.534 & $-0.59$   & 0.033  & 0.050  & $3.95$  \\
    CZAR ($\alpha=0.1$)   & 0.513 & 0.514 & 0.511 & $-0.58$   & 0.025  & 0.019  & $1.15$  \\
    CZAR ($\alpha=0.5$)   & 0.521 & 0.523 & 0.532 & $-0.71$   & 0.033  & 0.029  & $4.76$  \\
    CZAR ($\alpha=1$)     & 0.514 & 0.514 & 0.527 & $-0.82$   & 0.055  & 0.032  & $5.08$  \\
    CZAR (tuned, $\alpha\approx 0.054$) & 0.524 & 0.519 & 0.561 & $-0.60$ & 0.039 & 0.056 & $7.09$ \\
    \hline
    \multicolumn{8}{l}{\textit{BTC 1-hour}}\\
    \hline
    L1                    & 0.518 & 0.505 & 0.473 & $-0.96$   & 0.019  & 0.053  & $-1.09$ \\
    L2                    & 0.517 & 0.513 & 0.485 & $-1.02$   & 0.001  & 0.021  & $1.29$  \\
    CZAR ($\alpha=0.005$) & 0.521 & 0.526 & 0.538 & $-0.41$   & 0.038  & 0.046  & $2.67$  \\
    CZAR ($\alpha=0.01$)  & 0.526 & 0.528 & 0.524 & $-0.54$   & 0.034  & 0.065  & $2.38$  \\
    CZAR ($\alpha=0.05$)  & 0.516 & 0.522 & 0.531 & $-0.49$   & 0.034  & 0.058  & $2.86$  \\
    CZAR ($\alpha=0.1$)   & 0.526 & 0.546 & 0.550 & $-0.55$   & 0.047  & 0.072  & $5.96$  \\
    CZAR ($\alpha=0.5$)   & 0.515 & 0.521 & 0.520 & $-0.71$   & 0.044  & 0.059  & $3.89$  \\
    CZAR ($\alpha=1$)     & 0.512 & 0.523 & 0.529 & $-0.62$   & 0.041  & 0.044  & $3.21$  \\
    CZAR (tuned, $\alpha\approx 0.0045$) & 0.519 & 0.525 & 0.520 & $-0.45$ & 0.031 & 0.052 & $3.00$ \\
    \hline
    \end{tabular}
    \caption{Held-out metrics for LightGBM models trained with each loss on BTC log-returns, at the 15-minute (top) and 1-hour (bottom) horizons. DA$_{1\sigma}$ and DA$_{\mathrm{IQR}}$ denote directional accuracy restricted to truths with $|z|>1\sigma$ and to truths outside the interquartile range, respectively. The symmetric losses L1 and L2 shrink their predictions well below the truth ($\log_{10}\mathrm{AR}\approx-0.9$ to $-1.2$, i.e.\ roughly an order of magnitude too small) and their directional accuracy on large moves sits at or below the $50\%$ chance line. Every CZAR variant keeps $\log_{10}\mathrm{AR}$ closer to zero, and often raise DA on large moves above chance, lift $R_P$ and IC several-fold over the symmetric baselines, and generally improve the naive Sharpe. Full-sample DA differences are smaller and often within sampling error. Binomial standard errors on the $2\,000$-candle test window are $\approx 1.1$ percentage points on the full-sample DA, and $\approx 2$ and $\approx 1.6$ points on DA$_{1\sigma}$ and DA$_{\mathrm{IQR}}$, which use only the roughly one-third and one-half of truths in the tails. The naive Sharpe ratios exclude costs (i.e. are not comparable to live trading) and are annualized from test windows of $\approx 21$ days (15-minute) and $\approx 83$ days (1-hour).}
    \label{tab:model-training}
\end{table}
    
In this section, we turn from idealized Gaussian predictions to a gradient-boosted model trained on real data.
We test a LightGBM regressor predicting BTC log-returns at two intraday horizons (15-minute and 1-hour candles), with CZAR as the training objective through its analytical gradient and Hessian (\S\ref{sec:czar:derivatives}).
This experiment is intended as a controlled illustration of how the choice of training loss affects what a model learns, not as a search for the best attainable forecaster; the absolute numbers matter far less than the comparison across losses.
The feature set is intentionally restricted to a basic, easy-to-reproduce set, so the question is which loss extracts the most usable signal from the simplified set, rather than whether the features are competitive in absolute terms.
The learning rate is also calibrated per loss and then held fixed (as described below), factoring out the strong dependence of the effective step size on the loss function.

For the CZAR loss we sweep $\alpha$ from $0.005$ to $1$ and additionally report a CZAR-tuned variant (where $\alpha$ is treated as an Optuna parameter, but using fixed $\alpha=1$ for evaluation to keep the loss scale constant).
The CZAR loss variants are compared to two symmetric baseline loss functions (L1, L2).
Table~\ref{tab:model-training} summarizes the held-out test set at each horizon across the prediction-quality metrics introduced in \S\ref{sec:evaluation}: directional accuracy~(DA) over the full sample and restricted to large-magnitude truths (those with $|z|>1\sigma$, and separately those outside the interquartile range), log aspect ratio $\log_{10}(\sigma_{\hat y}/\sigma_y)$, Pearson correlation $R_P$, rank correlation (IC), and a naive (annualized) Sharpe ratio\footnote{Each period the strategy takes a unit long position when $\hat y>0$ and a unit short position when $\hat y<0$, earning the realized return $\operatorname{sign}(\hat y)\,y$. The reported figure is the ratio of the mean of $\operatorname{sign}(\hat y)\,y$ to its standard deviation, annualized by $\sqrt{T}$ (the annualization factor is given below). We call it \emph{naive} because it uses only the sign of the prediction (discarding magnitude and confidence), stays fully invested and rebalances every candle, and excludes transaction costs, slippage, and position sizing.}.

Inputs are one-minute BTC/USD OHLCV candles sourced from Tiingo\footnote{\url{https://www.tiingo.com/}} (retrieved on 2026-04-30), aggregated to the prediction horizon.
The target is the next aggregated-candle log return $\log(c_{t+1}/c_t)$.
For the CZAR objective and evaluation, returns are standardized by a trailing volatility $\sigma$, computed using a $100$-candle rolling standard deviation of single-candle realized returns (minimum $50$ candles).
We use a single most-recent split at each horizon of $20\,000$ candles for training (with the last $2\,000$ held out as Optuna inner validation), a one-candle gap, then a $2\,000$-candle test window.
Features are computed exclusively from past candles: single-candle log returns and within-candle realized volatilities at the six most recent lags; candle-geometry features (the log high--low range, the open-to-close return, the upper and lower wick lengths, and the position of the close within the candle range); the log gap between the close and the OHLC4 typical price together with its $5$- and $15$-candle rolling means; volume features (log volume and two log volume-change ratios over horizon-relative windows); and $\sin$/$\cos$ encodings of time-of-day and day-of-week (from UTC timestamps).

Each entry in Table~\ref{tab:model-training} is independently tuned with Optuna (TPE sampler, $300$ trials per model), minimizing the sample mean of the logarithm of the per-sample values of the model's own training loss, $\langle\log|\mathcal{L}|\rangle$.
The search is joint over the LightGBM hyperparameters \texttt{num\_leaves}~$\in[15,127]$, \texttt{min\_child\_samples}~$\in[20,300]$, \texttt{subsample} and \texttt{colsample\_bytree}~$\in[0.6,1]$, and the $L^1$ and $L^2$ leaf regularization strengths~$\in[10^{-8},10]$ (log-scaled), with a fixed random seed ($42$) so that each per-row Optuna study is deterministic.
The training window is fixed at $20\,000$ candles.
Rather than tuning the number of trees, each fit early-stops after $100$ rounds with no inner-validation gain (capped at $5\,000$, though in practice it is typically $<200$), and the stopped iteration count sets the final tree count.
The learning rate requires care because different losses produce Newton steps of very different magnitude.
A fixed rate would place L2 and small-$\alpha$ CZAR at very different points on the convergence curve.
We calibrate a per-loss scale $s_\mathrm{loss}$ by fitting a one-tree LightGBM with \texttt{learning\_rate}~$=1$ for each loss, measuring the standard deviation of its first-step predictions, and dividing by the L2 reference; the learning rate is then fixed at $0.05/s_\mathrm{loss}$, anchored to a base L2 rate of $0.05$.
For the CZAR-tuned row, $\alpha\in[10^{-3},1]$ is sampled on a log grid as an additional Optuna parameter and the scale is interpolated between calibrated grid points.
Custom losses are passed to LightGBM through its \texttt{objective} hook (using the closed-form gradient and Hessian from \S\ref{sec:czar:derivatives}) and the Hessian is clipped from below at $10^{-6}$ to guard against degenerate Newton steps when $\beta_\mathrm{eff}$ is very small.
The naive Sharpe ratios are annualized by $\sqrt{525\,960/\ell}$, where $\ell$ is the prediction horizon in minutes and $525\,960$ is the number of one-minute intervals per year, inferred from the median candle interval; this gives factors of $\approx 187$ at the 15-minute horizon and $\approx 94$ at the 1-hour horizon. The $\sqrt{T}$ scaling assumes serially uncorrelated returns and should be read as a reporting convention rather than an estimate of attainable annual performance \citep{lo02}.

The qualitative picture matches the idealized analysis of \S\ref{sec:eval:idealized}.
The two symmetric losses exhibit the aspect-ratio degeneracy.
At both horizons their predictions are shrunk roughly an order of magnitude below the truth ($\log_{10}\mathrm{AR}\approx-0.9$ to $-1.2$), their full-sample DA hovers near the chance line, and their directional accuracy on large-magnitude moves falls \emph{at or below} $50\%$ (DA$_{1\sigma}\approx 0.47$--$0.50$).
In contrast, CZAR keeps the aspect ratio within a factor of a few of unity across the whole $\alpha$ sweep and lifts DA$_{1\sigma}$ above chance for every setting ($\approx 0.50$--$0.56$), while raising $R_P$ and IC and generally improving the naive Sharpe.

However, the improvement for CZAR is not uniform across every metric.
Given a binomial standard error of $\approx 1.1$ percentage points (pp) on the $2\,000$-candle test window, full-sample DA is generally statistically indistinguishable from L1 and L2.
Instead, the improvement in CZAR tends to occur for larger moves (DA$_{1\sigma}$ and DA$_{\mathrm{IQR}}$), which also help to drive the several-fold gain in $R_P$ and higher naive Sharpe.
The conditional accuracies DA$_{1\sigma}$ and DA$_{\mathrm{IQR}}$ are computed on tail subsets (roughly a third and a half of the sample) and so carry larger standard errors ($\approx 2$ and $\approx 1.6$\,pp), but the sign of the gap (symmetric losses at or below chance, CZAR above) is consistent across the two horizons and both conditioning rules.

Every fixed CZAR setting clears both symmetric losses (L1 and L2) on the naive Sharpe at the 1-hour horizon, and all but one ($\alpha=0.01$) do so at the 15-minute horizon.
The tuned variant lands at a small $\alpha$ at both horizons ($\approx 0.054$ and $\approx 0.0045$) and delivers the best 15-minute DA$_{1\sigma}$ ($0.561$) while keeping the aspect ratio well-behaved.
The log aspect ratio itself varies systematically with $\alpha$.
It grows as $\alpha$ decreases, from $\approx-0.82$ at $\alpha=1$ to $\approx-0.46$ at $\alpha=0.005$ at the 15-minute horizon (and the same trend, more noisily, at the 1-hour horizon), because a smaller $\alpha$ produces larger Newton steps in the undershoot/wrong-direction region and so pushes predictions further from zero (Figure~\ref{fig:czar-grad-panel}).

A final check isolates where the improvement comes from.
The metrics in Table~\ref{tab:model-training} tune and early-stop each model on its own log-loss (the CZAR log-loss for the CZAR rows, the symmetric log-loss for L1 and L2).
Instead, we repeat the 15-minute experiment selecting \emph{every} model, including the CZAR-trained ones, on the symmetric log-L1 validation loss instead; the boosting gradients are unchanged, so only the model-selection and early-stopping criterion differs.
In this test, the CZAR models are pulled back onto the aspect-ratio degeneracy.
Their log aspect ratios collapse from the $-0.4$ to $-0.8$ range into the $-1.0$ to $-1.4$ range of the symmetric losses (predicted spreads shrink from roughly a quarter to a tenth of the true scale), and DA$_{1\sigma}$ and DA$_{\mathrm{IQR}}$ are reduced to $\approx 50$\%.
Training with CZAR is thus necessary but not sufficient (a symmetric selection criterion silently reintroduces the shrinkage attractor even when the gradients are CZAR's), so the gains reported above rely on using CZAR as both the training objective and the model-selection loss.

%%%%%%%%%%%%%%%%%%%%%%%%%%%%%%%%%%%%
\section{Discussion and conclusions}
\label{sec:discussion}

We have introduced CZAR, a composite zero-agnostic loss function for log-return prediction, motivated by a simple but consequential failure of standard regression losses: because the conditional mean of returns sits close to zero, symmetric losses make the constant zero forecast a near-optimal solution, an attractor during training and a misleading benchmark during evaluation.
From an analysis of Gaussian linear prediction models we derived five requirements for an adequate returns loss (\S\ref{sec:background}) and constructed a piecewise quadratic loss satisfying all of them (\S\ref{sec:czar}).
CZAR is provably convex in the prediction at fixed true value, has a closed-form gradient and Hessian suitable for gradient-boosted libraries, and its four hyperparameters reduce in practice to a single choice of $\alpha$ through the correlated defaults ($\beta^*$ and $C^*$) and the fixed hinge width ($\tau$).
In idealized tests its log-loss breakeven directional accuracy tracks the $50\%$ chance line where all symmetric monotonic losses share a sharply-increasing universal envelope, an advantage that survives heavy-tailed truth distributions.
The correlated defaults $\beta^*(\alpha)$ and $C^*(\alpha)$ were calibrated against the Gaussian linear testbed; under heavy-tailed truths the log-loss breakeven degrades mildly, and recalibrating against the target distribution should recover much of the gap (\S\ref{sec:eval:idealized}).

In a LightGBM experiment on intraday (15-minute and 1-hour) BTC log-returns, CZAR-trained models reduce the prediction shrinkage exhibited by the L1 and L2 baselines and improve the naive (cost-excluded) Sharpe ratio and the directional accuracy on large-magnitude returns (\S\ref{sec:evaluation}).
The loss floor is invisible to per-sample gradients, so the training-time pressure away from the zero attractor rests entirely on the base asymmetry; the loss floor acts only through sample-averaged validation and ranking losses.

Beyond its role as a training objective, the sample-averaged CZAR log-loss is a candidate criterion for ranking or scoring forecasters.
Two properties from \S\ref{sec:eval:idealized} support this.
First, across the $(\rho,\sigma_N)$ sweep the positive-signal models collapse onto a single DA--loss track.
Models with the same directional accuracy receive essentially the same CZAR loss regardless of how that accuracy is composed from signal and noise, so ranking by mean CZAR loss approximates ranking by directional skill rather than by residual magnitude.
Second, the loss floor pins the loss of near-zero predictors to at least $C$, so trivial forecasters cannot lead a ranking, while wrong-direction forecasters remain penalized through the Region~A asymmetry.
These properties matter wherever sample-averaged losses carry decision weight, such as in conventional pipelines, for hyperparameter tuning, early stopping, and model selection (where selecting on a symmetric validation loss can silently reintroduce the zero-returns attractor even when training uses CZAR), and in decentralized inference networks such as Allora \citep{kruijssen24}, where topic-level losses determine the relative rewards of competing workers and a symmetric criterion would actively subsidize zero-like inference strategies.

The evaluation results also carry a methodological warning that is independent of CZAR itself.
Symmetric losses conflate two distinct properties of a forecaster, i.e. directional skill and prediction variance.
A model can reduce its MAE or MSE either by becoming more accurate or simply by shrinking the magnitude of its predictions toward the unconditional mean; the L1- and L2-trained models in Table~\ref{tab:model-training}, whose predictions are shrunk roughly an order of magnitude below the truth while their directional accuracy on large moves sits at or below chance, are the practical case of a low-loss but directionally uninformative forecaster.
Evaluation methodology should therefore treat the aspect ratio $\sigma_{\hat{y}}/\sigma_y$ as a first-class diagnostic alongside any loss-based leaderboard.
It should also recognize the averaging convention as part of the loss design, because linear and log averaging induce different breakeven behavior, and asymmetries that are safe under one can become pathological under the other (\S\ref{sec:background}).
Loss-based rankings constructed from symmetric criteria, common in forecasting competitions and production model selection alike, will systematically favor shrunken or degenerate predictors over noisy but directionally informative ones.

Tests of the CZAR loss in this work are naturally non-exhaustive.
However, the central message is nevertheless robust across both the idealized and the trained experiments: when the conditional mean of the target is near zero, the loss function is not a neutral implementation detail but the decisive design choice.
CZAR offers a convex, trainable, and evaluation-consistent loss that removes the zero-returns attractor by construction, at the cost of a characterizable forecast bias.
We argue this compromise is favorable wherever directional information is the economically relevant output.

% Optional acknowledgments section: thank funding sources or contributors who
% do not meet authorship criteria. Uncomment and edit if needed.
% \section*{Acknowledgments}
%

\begin{sloppypar}
\bibliographystyle{ADI}
{\small
\bibliography{bibliography} % There is a file named "bibliography.bib" in the root directory of the repository. Please add your own bibliographic entries there.
}
\end{sloppypar}

%%%%%%%%%%%%%%%%%%
\begin{appendices}

\counterwithin*{figure}{section}
\renewcommand{\thefigure}{\thesection\arabic{figure}}
\renewcommand{\theHfigure}{\thesection\arabic{figure}}

\counterwithin*{equation}{section}
\renewcommand{\theequation}{\thesection\arabic{equation}}
\renewcommand{\theHequation}{\thesection\arabic{equation}}

\section{Reference implementation}
\label{app:code}

The following Python listing implements the base CZAR loss (\S\ref{sec:czar}), gradient, and Hessian using NumPy.
All functions operate on batches (\texttt{y\_true}, \texttt{y\_pred}, and \texttt{std} may be scalars or arrays of matching shape).
The listing includes two numerical guards that are absent from the mathematical formulation of \S\ref{sec:czar:floor}.
The loss-floor normalization is protected against division by zero (\texttt{safe\_norm}), and the loss-floor term is clipped at $10^{-12}$ to guard against floating-point underflow when logarithms of the total loss are taken for evaluation; neither affects the loss away from degenerate parameter values.
The Hessian clip at $10^{-6}$ used in the LightGBM experiments (\S\ref{sec:eval:training}) is applied in the training wrapper, not in this reference implementation.

\begin{lstlisting}
import numpy as np

ALPHA = 1.0   # quadratic regularization strength
TAU   = 0.5   # hinge smoothing width


def _correlated_beta(alpha):
    # Two-term power-law fit to optimal (alpha, beta*) pairs from the breakeven
    # DA heatmap; minimizes the DA required to beat a zero-returns predictor.
    return 4.2 * alpha**0.56 + 27.0 * alpha**2.17


def _correlated_C(alpha):
    # Asymmetric Hill-bell fit (with linear correction) to optimal C per alpha
    # at fixed tau; keeps the log-loss breakeven DA at or above 50% across
    # noise scales, and satisfies C*(0) = 0.
    A, t, p, k, R, m = 7.9, 0.00459, 0.657, 0.684, 2.25, -0.218
    h = alpha**p / (t**p + alpha**p)
    return A * h**k * (1 - h) + (R + m * alpha) * h


def czar_loss(y_true, y_pred, std, mean=0,
              alpha=ALPHA, beta=None, C=None, tau=TAU):
    if beta is None:
        beta = _correlated_beta(alpha)
    if C is None:
        C = _correlated_C(alpha)
    # Standardize to zero mean, unit variance.
    z     = (y_true - mean) / std
    z_hat = (y_pred - mean) / std

    # Unsigned prediction error and true-value magnitude.
    dz = np.abs(z_hat - z)
    a  = np.abs(z)

    # s is the sign of the true value (convention: +1 when z = 0).
    # u = s * z_hat projects the prediction onto the direction of the truth:
    #   u > a  means the prediction overshoots (Region B),
    #   u <= a means it undershoots or is in the wrong direction (Region A).
    s = np.where(z == 0, 1.0, np.sign(z))
    u = s * z_hat

    # Adaptive discount: b_eff -> 1 as a -> 0 (both regions become quadratic),
    # b_eff -> 0 as a -> inf (Region B vanishes, Region A becomes MAE-like).
    b_eff = 1.0 / (1.0 + beta * a)

    # Region A: MAE + quadratic.  Region B: discounted quadratic only.
    L_A = (1 - b_eff) * dz + 0.5 * alpha * dz**2
    L_B = b_eff * 0.5 * alpha * dz**2

    # --- Loss floor ---
    # L0 is the base loss when z_hat = 0 (always in Region A since u = 0 <= a).
    # deficit = C - L0: positive when the zero-prediction loss is below the
    # target C, meaning additional floor loss is required.
    L0      = (1 - b_eff) * a + 0.5 * alpha * a**2
    deficit = C - L0

    # hinge(x, tau) ~ max(x, 0) for small tau (smooth rectification).
    # Normalized by hinge(C, tau) so that L_floor = C when deficit = C
    # (i.e. at a = 0, where L0 = 0) and L_floor -> 0 when L0 >> C.
    hinge     = 0.5 * (deficit + np.sqrt(deficit**2 + tau**2))
    norm      = 0.5 * (C       + np.sqrt(C**2       + tau**2))
    safe_norm = np.where(norm > 0, norm, 1.0)
    L_floor   = np.where(C > 0, np.maximum(C * hinge / safe_norm, 1e-12), 0.0)

    return np.where(u > a, L_B, L_A) + L_floor


def czar_gradient(y_true, y_pred, std, mean=0,
                  alpha=ALPHA, beta=None):
    if beta is None:
        beta = _correlated_beta(alpha)
    # L_floor depends only on z (not z_hat), so it contributes no gradient.
    z     = (y_true - mean) / std
    z_hat = (y_pred - mean) / std
    err   = z_hat - z       # signed error
    dz    = np.abs(err)
    a     = np.abs(z)
    s     = np.where(z == 0, 1.0, np.sign(z))
    u     = s * z_hat

    # s_err = sign of the error, used to direct the gradient.
    # Set to 0 exactly at z_hat = z so the gradient is 0 at the minimum,
    # consistent with Region B where G_B -> 0 as dz -> 0.
    s_err = np.where(err == 0, 0.0, np.sign(err))
    b_eff = 1.0 / (1.0 + beta * a)

    # Region A gradient: magnitude (1 - b_eff) + alpha*dz, directed by s_err.
    # Approaches -(1-b_eff)/std as z_hat -> z from below (the gradient step),
    # then 0 from Region B, concentrating reward near the true value.
    G_A = s_err * ((1 - b_eff) + alpha * dz)
    G_B = s_err * b_eff * alpha * dz

    return np.where(u > a, G_B, G_A) / std


def czar_hessian(y_true, y_pred, std, mean=0,
                 alpha=ALPHA, beta=None):
    if beta is None:
        beta = _correlated_beta(alpha)
    z     = (y_true - mean) / std
    z_hat = (y_pred - mean) / std
    a     = np.abs(z)
    s     = np.where(z == 0, 1.0, np.sign(z))
    u     = s * z_hat
    b_eff = 1.0 / (1.0 + beta * a)

    # Region A Hessian: the MAE term (1-b_eff)*dz has zero curvature,
    # so only the quadratic term contributes -> constant alpha.
    # Region B Hessian: b_eff * alpha, which shrinks toward 0 for large |z|.
    # Both are non-negative, confirming convexity in y_pred at fixed y_true.
    H_A = alpha
    H_B = b_eff * alpha

    return np.where(u > a, H_B, H_A) / std**2
\end{lstlisting}

\section{Breakeven directional accuracy}
\label{app:breakeven}

\subsection*{Prediction model and directional accuracy}

We model predicted returns as a linear combination of signal and noise:
\begin{equation}
    \hat{y} = \rho\,y + \sigma_N\,\xi,
    \qquad y \sim \mathcal{N}(0,\sigma^2),
    \quad \xi \sim \mathcal{N}(0,1) \text{ independent},
    \label{eq:pred-model}
\end{equation}
where $\rho \in \mathbb{R}$ is the signal coefficient and $\sigma_N \ge 0$ is the noise scale. The directional accuracy (DA) is the probability of a correct-sign prediction:
\begin{equation}
    \mathrm{DA}(\rho, \sigma_N) = P\!\left(\operatorname{sign}(\hat{y}) = \operatorname{sign}(y)\right).
    \label{eq:da-def}
\end{equation}
The pair $(y, \hat{y})$ is jointly bivariate Gaussian with zero mean, $\mathrm{Cov}(y,\hat{y}) = \rho\sigma^2$, and correlation
\begin{equation}
    r = \frac{\rho\sigma}{\sqrt{\rho^2\sigma^2 + \sigma_N^2}}.
\end{equation}
The sign-concordance formula for zero-mean bivariate Gaussians gives $P(X>0,\,Y>0) = \tfrac{1}{4} + \tfrac{\arcsin r}{2\pi}$, so
\begin{equation}
    \mathrm{DA} = 2\,P(y>0,\,\hat{y}>0) = \frac{1}{2} + \frac{\arcsin r}{\pi}.
\end{equation}
Applying the identity $\arcsin\!\left(x/\sqrt{x^2+1}\right) = \arctan(x)$:
\begin{equation}
    \mathrm{DA}(\rho,\sigma_N) = \frac{1}{2} + \frac{1}{\pi}\arctan\!\left(\frac{\rho\sigma}{\sigma_N}\right).
    \label{eq:da-closed}
\end{equation}
The DA depends only on the signal-to-noise ratio $\mathrm{SNR} = \rho\sigma/\sigma_N$, not on $\rho$ and $\sigma_N$ separately.

\subsection*{Breakeven conditions}

A predictor \emph{breaks even} against a constant-zero forecast when its expected loss equals the baseline loss, either under linear or log averaging:
\begin{align}
    \text{linear:} &\quad \mathbb{E}\!\left[\mathcal{L}(y,\hat{y})\right] = \mathbb{E}\!\left[\mathcal{L}(y,0)\right], \label{eq:be-lin} \\
    \text{log:}    &\quad \mathbb{E}\!\left[\log\mathcal{L}(y,\hat{y})\right] = \mathbb{E}\!\left[\log\mathcal{L}(y,0)\right]. \label{eq:be-log}
\end{align}
For each noise scale $\sigma_N$, the minimum DA satisfying these conditions traces the breakeven curves in Figure~\ref{fig:breakeven_DA}.

\subsection*{Symmetric losses share a universal breakeven curve}

\begin{proposition}\label{prop:symm-breakeven}
For any symmetric loss $\mathcal{L}(y,\hat{y}) = g(|y-\hat{y}|)$ with $g$ strictly increasing and positive, and with $\mathbb{E}[g(|e|)]$ and $\mathbb{E}\!\left[\,|\log g(|e|)|\,\right]$ finite for every centered Gaussian~$e$, the linear and log breakeven conditions~\eqref{eq:be-lin}--\eqref{eq:be-log} are equivalent under model~\eqref{eq:pred-model}, and both reduce to
\begin{equation}
    (1-\rho)^2\sigma^2 + \sigma_N^2 = \sigma^2.
    \label{eq:symm-breakeven}
\end{equation}
\end{proposition}
In particular, the condition does not depend on $g$.

\begin{proof}
The prediction error is $e = \hat{y} - y = (\rho-1)y + \sigma_N\xi \sim \mathcal{N}(0,\,s^2)$ with $s^2 = (1-\rho)^2\sigma^2 + \sigma_N^2$. Writing $e = sZ$ with $Z\sim\mathcal{N}(0,1)$, the loss becomes $g(s|Z|)$. Since $g$ is strictly increasing, both
\begin{equation*}
    h_{\mathrm{lin}}(s) = \mathbb{E}[g(s|Z|)] \qquad \text{and} \qquad
    h_{\mathrm{log}}(s) = \mathbb{E}[\log g(s|Z|)]
\end{equation*}
are strictly monotone in $s$. Each breakeven condition therefore reduces to $s = \sigma$, which is~\eqref{eq:symm-breakeven}.
\end{proof}

\noindent The function $g$ cancels entirely. MSE ($g(r)=r^2$), MAE ($g(r)=r$), Huber loss, and any other symmetric loss meeting these conditions therefore all satisfy the same breakeven condition~\eqref{eq:symm-breakeven}.
Solving Equation~\eqref{eq:symm-breakeven} for the minimum signal coefficient gives $\rho_\mathrm{min} = 1 - \sqrt{1-(\sigma_N / \sigma)^2}$, which substituted into Equation~\eqref{eq:da-closed} yields the universal breakeven curve quoted in \S\ref{sec:background}, $\mathrm{DA} = 0.5 + \arctan(\rho_\mathrm{min}\sigma/\sigma_N)/\pi$, valid for both linear and log averaging simultaneously.

\subsection*{Asymmetric losses can decouple the two breakevens}

Proposition~\ref{prop:symm-breakeven} establishes a hard lower bound: no symmetric loss function can achieve a breakeven curve below~\eqref{eq:symm-breakeven} in either linear or log averaging. For a loss to improve on this bound it must be asymmetric, i.e. it must treat $\operatorname{sign}(y) = \operatorname{sign}(\hat{y})$ differently from $\operatorname{sign}(y) \ne \operatorname{sign}(\hat{y})$.

For an asymmetric loss, the distribution of $\mathcal{L}(y,\hat{y})$ changes shape (not only scale) as $\rho$ varies, because the fraction of samples falling in the correct- and wrong-direction regions depends on $\rho$. The log and linear expectations then weight the distribution differently, and the two breakeven conditions are no longer equivalent.

CZAR exploits exactly this freedom by making its loss depend separately on $\operatorname{sign}(z)$ and $\operatorname{sign}(\hat{z})$, with different functional forms in each region.
As visible in Figure~\ref{fig:czar_breakeven}, CZAR's log-loss breakeven lies substantially below the symmetric bound~\eqref{eq:symm-breakeven}, while its linear-loss breakeven tracks that bound more closely.
The separation between the two breakevens, achievable only through asymmetry, is the central property that motivates the CZAR design.

\section{Pseudo-Huber smoothing}
\label{app:smoothing}

The core formulation of \S\ref{sec:czar} uses a plain MAE distance in Region~A. This produces a step discontinuity in the gradient at $\hat{z}=z$ of magnitude $(1-\beta_{\mathrm{eff}})/\sigma$ (\S\ref{sec:czar:derivatives}). Setting $\varepsilon > 0$ replaces the MAE term with a pseudo-Huber distance, smoothing the step into a continuous zero-crossing. This appendix describes the modifications required.

\subsection*{Modified quantities}

Introduce a smoothing scale $\varepsilon_{\mathrm{eff}}(|z|) = \varepsilon |z|$ proportional to the true-value magnitude, and define the pseudo-Huber effective distance
\begin{equation}
    \Delta z_{\mathrm{eff}} = \sqrt{1+\varepsilon^2}\Bigl(\sqrt{\Delta z^2 + \varepsilon_{\mathrm{eff}}^2} - \varepsilon_{\mathrm{eff}}\Bigr).
    \label{eq:dzeff}
\end{equation}
When $\varepsilon=0$, $\Delta z_{\mathrm{eff}} = \Delta z$ (MAE). For small errors ($\Delta z \ll \varepsilon_{\mathrm{eff}}$), $\Delta z_{\mathrm{eff}} \approx \sqrt{1+\varepsilon^2}\,\Delta z^2/(2\varepsilon |z|)$ (quadratic); for large errors ($\Delta z \gg \varepsilon_{\mathrm{eff}}$), $\Delta z_{\mathrm{eff}} \approx \sqrt{1+\varepsilon^2}\,\Delta z$ (linear, with boosted slope). Because $\Delta z_{\mathrm{eff}} < \Delta z$ for small errors at fixed $\varepsilon > 0$, an amplitude boost factor
\begin{equation}
    b(|z|) = 1 + 2\varepsilon\bigl(1-\beta_{\mathrm{eff}}(|z|)\bigr) \ge 1
    \label{eq:boost}
\end{equation}
is introduced to compensate.
Without it, increasing $\varepsilon$ would reduce the Region~A loss for moderate errors relative to the $\varepsilon=0$ case.

\subsection*{Modified base loss}

Replace $(1-\beta_{\mathrm{eff}})\,\Delta z$ in Region~A of Equation~\eqref{eq:czar-base} with $b\,(1-\beta_{\mathrm{eff}})\,\Delta z_{\mathrm{eff}}$:
\begin{equation}
    \mathcal{L}_{\mathrm{base}}^{(\varepsilon)}(z, \hat{z}) =
    \begin{cases}
        b\,(1-\beta_{\mathrm{eff}})\,\Delta z_{\mathrm{eff}} + \dfrac{\alpha}{2}\,\Delta z^2
            & \text{(Region A)}, \\[10pt]
        \beta_{\mathrm{eff}}\,\dfrac{\alpha}{2}\,\Delta z^2
            & \text{(Region B)}.
    \end{cases}
    \label{eq:czar-base-eps}
\end{equation}
Region~B is unchanged.

\subsection*{Modified gradient and Hessian}

The Region~A gradient becomes
\begin{equation}
    \frac{\partial \mathcal{L}_{\mathrm{base}}^{(\varepsilon)}}{\partial \hat{y}}\Bigg|_{\text{Region A}}
    = \frac{s_\Delta}{\sigma}\left[
        b\,(1-\beta_{\mathrm{eff}})\,\sqrt{1+\varepsilon^2}\,\frac{\Delta z}{\sqrt{\Delta z^2+\varepsilon_{\mathrm{eff}}^2}}
        + \alpha\,\Delta z
    \right].
    \label{eq:grad-eps}
\end{equation}
The factor $\Delta z / \sqrt{\Delta z^2 + \varepsilon_{\mathrm{eff}}^2} \to 0$ as $\Delta z \to 0$ (for $\varepsilon > 0$), so the gradient vanishes continuously at $\hat{z}=z$. The Region~A Hessian gains a curvature term:
\begin{equation}
    \frac{\partial^2 \mathcal{L}_{\mathrm{base}}^{(\varepsilon)}}{\partial \hat{y}^2}\Bigg|_{\text{Region A}}
    = \frac{1}{\sigma^2}\left[
        b\,(1-\beta_{\mathrm{eff}})\,\sqrt{1+\varepsilon^2}\,
        \frac{\varepsilon_{\mathrm{eff}}^2}{\bigl(\Delta z^2 + \varepsilon_{\mathrm{eff}}^2\bigr)^{3/2}}
        + \alpha
    \right].
    \label{eq:hess-eps}
\end{equation}
This is maximized at $\Delta z = 0$ (where it equals $\tfrac{1}{\sigma^2}\!\left[b(1-\beta_{\mathrm{eff}})\sqrt{1+\varepsilon^2}/\varepsilon_{\mathrm{eff}} + \alpha\right]$) and decays to $\alpha/\sigma^2$ as $\Delta z \to \infty$, recovering the constant Hessian of the $\varepsilon=0$ case. The Region~B gradient and Hessian are unchanged.

\subsection*{Modified loss floor}

At $\hat{z}=0$, the pseudo-Huber effective distance evaluates to
\begin{equation}
    \Delta z_{0,\mathrm{eff}} = |z|\bigl(1 + \varepsilon^2 - \varepsilon\sqrt{1+\varepsilon^2}\bigr),
    \label{eq:dz0eff}
\end{equation}
obtained by substituting $\Delta z = |z|$ and $\varepsilon_{\mathrm{eff}} = \varepsilon |z|$ into Equation~\eqref{eq:dzeff}. The zero-prediction base loss in Equation~\eqref{eq:L0} becomes
\begin{equation}
    \mathcal{L}_0^{(\varepsilon)}(z) = b\,(1-\beta_{\mathrm{eff}})\,\Delta z_{0,\mathrm{eff}} + \frac{\alpha}{2}\,z^2,
    \label{eq:L0-eps}
\end{equation}
which replaces $(1-\beta_{\mathrm{eff}})\,|z|$ from the $\varepsilon=0$ case. The loss-floor formulas~\eqref{eq:hinge}--\eqref{eq:lfloor} are otherwise unchanged, with the deficit redefined as $d(z) = C - \mathcal{L}_0^{(\varepsilon)}(z)$.

\subsection*{Convexity}

The Region~A Hessian~\eqref{eq:hess-eps} satisfies $H_A \ge \alpha > 0$ and the Region~B Hessian remains $\beta_{\mathrm{eff}}\alpha > 0$ (since $\alpha>0$ and $\beta_{\mathrm{eff}}\in(0,1]$).
For $\varepsilon>0$ the gradient now vanishes continuously from both sides at $\hat{z}=z$ (no step discontinuity), and (writing $s=\operatorname{sign}(z)$) has sign $-s$ on the Region~A side of the region boundary and $+s$ on the Region~B side (i.e.\ it is negative for $\hat{z}<z$ and positive for $\hat{z}>z$ when $z>0$, with the sides reversed when $z<0$).
Convexity in $\hat{y}$ at fixed $y$ is preserved for all $\varepsilon \ge 0$.
On each side of the region boundary the derivative is continuous and non-decreasing in $\hat{y}$, and any residual one-sided jump at $\hat{z}=z$ ($\ge 0$ at $\varepsilon=0$, exactly $0$ for $\varepsilon>0$) is upward.

\section{Parameter optimization}
\label{app:optimization}

The prediction model and breakeven directional accuracy are defined in \S\ref{sec:background}.
Here we apply them to the CZAR loss to determine the correlated defaults $\beta^*(\alpha)$ and $C^*(\alpha)$, along with the constant hinge width $\tau$ adopted across the working range.

The floor parameters $C$ and $\tau$ are both optimized by the requirement that the log-loss breakeven DA stays at or above $50\%$ at all noise scales (see \S\ref{sec:background}).
A breakeven below $50\%$ would mean a model that predicts consistently in the wrong direction achieves a lower loss than a zero-returns predictor.
However, $C$ and $\tau$ act on the floor term in a degenerate way and cannot be optimized jointly from the breakeven condition alone.
We therefore fix $\tau$ first from an independent criterion, and then optimize $C$ at the chosen $\tau$.

\subsection*{Optimizing $\beta$}

\begin{figure}[t]
    \centering
    \includegraphics[width=0.65\textwidth]{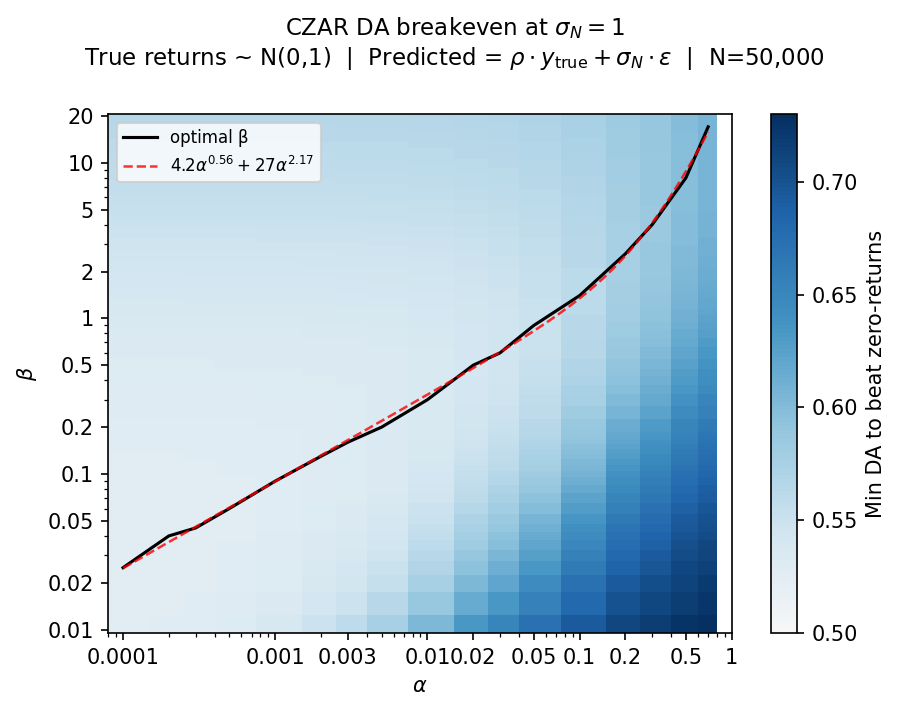}
    \caption{CZAR breakeven directional accuracy at $\sigma_N=1$ as a function of $\alpha$ and $\beta$ (both axes log-scale). Color indicates the minimum DA required to beat a zero-returns predictor under linear loss averaging; lower (redder) is better. The black curve marks the optimal $\beta^*(\alpha)$ and the red dashed curve the two-term power-law fit from Equation~\eqref{eq:correlated-beta}. The shallow basin around the optimum shows that the default is robust to moderate departures from $\beta^*$.}
    \label{fig:czar-beta-opt}
    \end{figure}
    
The breakeven DA is evaluated numerically at noise scale $\sigma_N = 1$ over $N=50{,}000$ samples from the Gaussian linear model ($y \sim \mathcal{N}(0,1)$, $\hat{y} = \rho\,y + \sigma_N\,\xi$).
Figure~\ref{fig:czar-beta-opt} shows the resulting breakeven DA as a joint function of $\alpha$ and $\beta$.
For each $\alpha$, the value $\beta^*(\alpha)$ that minimizes the breakeven DA is traced by the black curve.
The loss surface is shallow near this optimum, i.e. neighboring $(\alpha,\beta)$ pairs achieve nearly identical breakeven DA, so the default is robust to the exact gridding choice for $\beta$.

The two-term power-law fit (Equation~\eqref{eq:correlated-beta}, red dashed curve) closely tracks the numerical optimum across more than three decades of $\alpha$.
We restrict the fit to $\alpha \le 0.7$.
Beyond this the breakeven basin in $\beta$ becomes very shallow and the optimum drifts toward arbitrarily large values, so a single $\beta^*(\alpha)$ ceases to be a meaningful target and any sufficiently large $\beta$ achieves consistently low breakeven DA.
For practical use we recommend $\alpha \le 1$, with $\beta^*$ taken from Equation~\eqref{eq:correlated-beta} (a mild extrapolation in $(0.7, 1]$).
What is also clear from the figure is that the breakeven directional accuracy converges for $\alpha \lesssim 0.005$, and therefore cannot be improved by using arbitrarily small $\alpha$ values.

\subsection*{Optimizing $\tau$}

We fix $\tau$ by requiring that the expected log-loss $E_{y}[\log_{10}\mathcal{L}(y,\hat{y})]$, integrated over $y\sim\mathcal{N}(0,1)$, is monotonic in $\hat{z}$ for $\hat{z}\ge 0$ (and by symmetry for $\hat{z}\le 0$). A non-monotonic profile would imply that some non-trivial predicted return is preferred by the loss in expectation, biasing the optimizer toward that scale.

\begin{figure}[t]
\centering
\includegraphics[width=\textwidth]{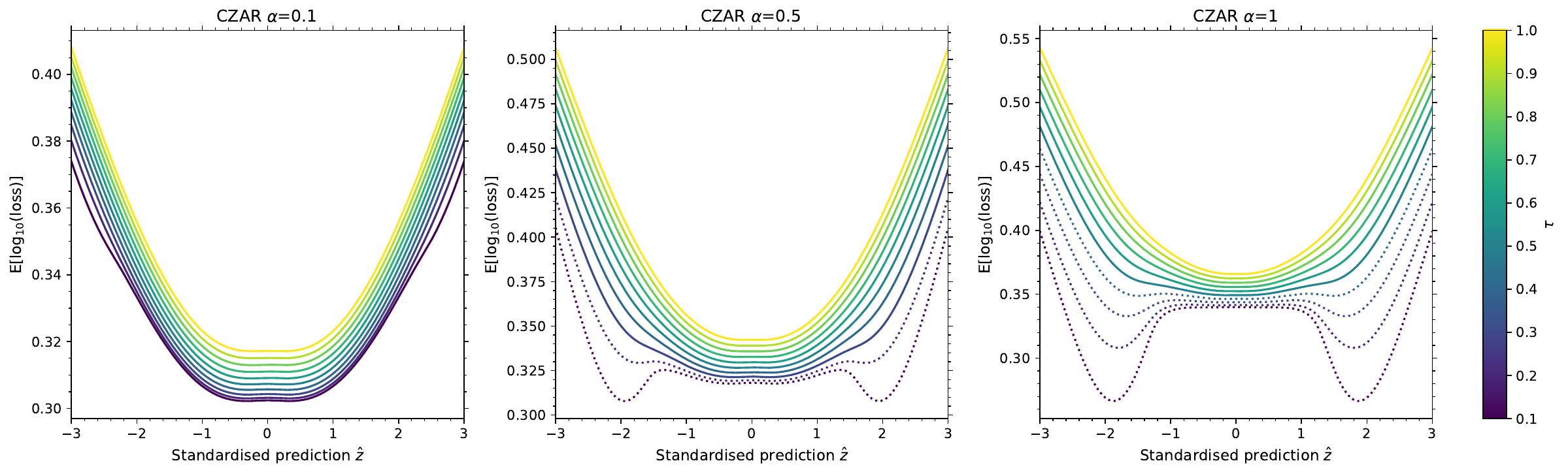}
\caption{Expected log-loss $E_{y\sim\mathcal{N}(0,1)}[\log_{10}\mathcal{L}(y,\hat{y})]$ as a function of standardized prediction $\hat{z}$ for $\tau\in[0.1,1.0]$ (color) at three $\alpha$ values, with $C=2$ and $\beta=\beta^*(\alpha)$. Solid curves are monotonic in $|\hat{z}|$ and define admissible $\tau$ values; dotted curves develop spurious minima near $\hat{z}\approx\pm 2$ and are excluded. The smallest admissible $\tau$ grows with $\alpha$; $\tau=0.5$ is admissible across $\alpha\le 1$ and is adopted as the default.}
\label{fig:czar-tau-opt}
\end{figure}

Figure~\ref{fig:czar-tau-opt} sweeps $\tau \in [0.1,1.0]$ at fixed $C=2$ and $\beta=\beta^*(\alpha)$ for three representative $\alpha$ values.
For each $\alpha$ there is a threshold $\tau_{\min}(\alpha)$ below which the curves develop shallow minima near $\hat{y}\approx\pm 2$ (dotted lines, excluded); for $\tau \ge \tau_{\min}(\alpha)$ the curves are monotonic (solid).
The threshold grows with $\alpha$.
At $\alpha=0.1$ values as low as $\tau\approx 0.1$ remain monotonic, whereas at $\alpha=1$ we require $\tau\gtrsim 0.5$.
We adopt $\tau=0.5$ as a constant default, i.e. the smallest value that remains monotonic across the working range $\alpha\le 1$.
Smaller $\tau$ could in principle be used at small $\alpha$, but the expected-loss profile is increasingly insensitive to $\tau$ there, so a single constant is preferred for simplicity.

\subsection*{Optimizing $C(\alpha)$}

\begin{figure}[t]
\centering
\includegraphics[width=\textwidth]{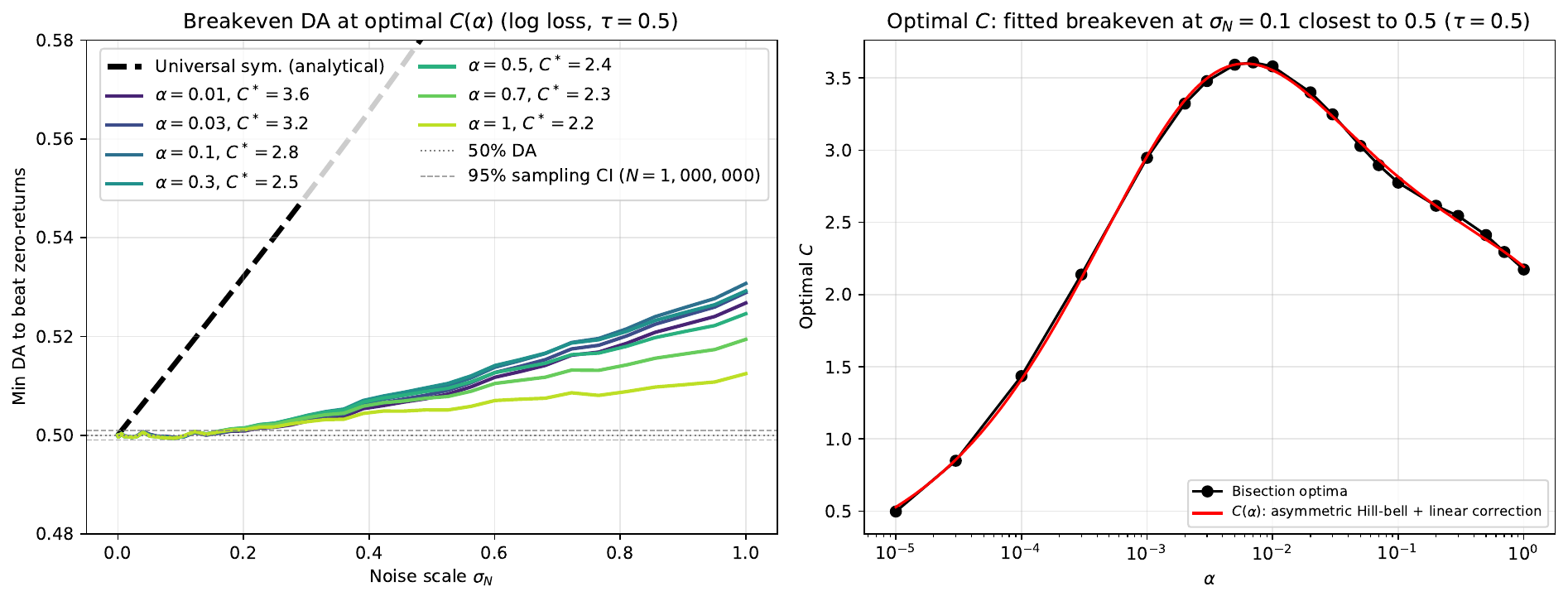}
\caption{Optimization of the loss floor parameter $C$. The left panel shows log-loss breakeven DA versus $\sigma_N$ at the optimized $C^*(\alpha)$ (color), with $\tau=0.5$ and $\beta=\beta^*(\alpha)$. The universal symmetric (MSE) breakeven (thick dashed black) is shown for reference; the dashed gray lines bound the $95\%$ binomial CI for $\mathrm{DA}=0.5$ at $N=10^6$. The right panel shows optimal $C^*(\alpha)$ (black points) and the asymmetric Hill-bell fit of Equation~\eqref{eq:correlated-C} (red line). $C^*$ is selected as the value for which a linear fit to the breakeven DA in a small window around $\sigma_N=0.1$ equals $0.5$.}
\label{fig:czar-C-opt}
\end{figure}

With $\tau$ fixed, $C$ is determined by requiring the log-loss breakeven DA to equal $50\%$ at a reference noise scale $\sigma_N=0.1$.
Operationally, for each $\alpha$ we evaluate the breakeven DA on a small $\sigma_N$ window around $0.1$, fit a line to those points, and bisect on $C$ until the linear fit at $\sigma_N=0.1$ equals $0.5$.
Targeting the fit rather than a single sample point suppresses binomial sampling noise (the $95\%$ CI for $\mathrm{DA}=0.5$ at $N=10^6$ is shown by the dashed gray lines in the left panel of Figure~\ref{fig:czar-C-opt}).
The breakeven curves at the resulting $C^*(\alpha)$ cluster tightly around $0.5$ across all $\sigma_N$, confirming that this single-point calibration produces a near-flat profile.

The right panel of Figure~\ref{fig:czar-C-opt} shows the resulting $C^*(\alpha)$.
The points rise from $C^*=0$ at $\alpha=0$ to a peak near $\alpha\sim 10^{-2}$, followed by a slow decline at larger $\alpha$.
We fit them with an asymmetric Hill-bell model with a linear correction, which captures both the zero limit and the high-$\alpha$ tail while remaining smooth across the transition.
The fitted coefficients give Equation~\eqref{eq:correlated-C}, repeated here for convenience:
\begin{equation*}
    C^*(\alpha) = A\, h(\alpha)^k\, [1-h(\alpha)] + (R + m\alpha)\, h(\alpha),
    \qquad h(\alpha) = \frac{\alpha^p}{t^p + \alpha^p},
\end{equation*}
with $(A, t, p, k, R, m) = (7.90,\, 4.59\times 10^{-3},\, 0.657,\, 0.684,\, 2.25,\, -0.218)$.

\end{appendices}

\end{document}